\documentclass[11pt]{article}
\usepackage[english]{babel}
\usepackage[utf8]{inputenc}
\usepackage{ifdraft}
\usepackage{amsmath}
\usepackage{amsfonts,amssymb}
\usepackage{pifont}% http://ctan.org/pkg/pifont
\newcommand{\cmark}{\ding{51}}%
\newcommand{\xmark}{\ding{55}}%
\usepackage{amsthm}
\usepackage{algorithm,xspace}
\usepackage[noend]{algpseudocode}
\usepackage{graphicx}
\usepackage{natbib}

\usepackage[symbol]{footmisc}

\usepackage{tablefootnote}

\usepackage{hyperref,xcolor}
\hypersetup{
	unicode=false,          % non-Latin characters in Acrobat?s bookmarks
	pdftoolbar=true,        % show Acrobat?s toolbar?
	pdfmenubar=true,        % show Acrobat?s menu?
	pdffitwindow=false,     % window fit to page when opened
	pdfstartview={FitH},    % fits the width of the page to the window
	pdftitle={},    % title
	pdfauthor={Soumendu Sundar Mukherjee},     % author
	pdfsubject={Mathematics},   % subject of the document
	pdfcreator={Creator},   % creator of the document
	pdfproducer={Producer}, % producer of the document
	pdfkeywords={keyword1} {key2} {key3}, % list of keywords
	pdfnewwindow=true,      % links in new window
	colorlinks=true,       % false: boxed links; true: colored links
	linktoc=page,
	linkcolor=blue,          % color of internal links
	citecolor=blue,        % color of links to bibliography
	filecolor=green,      % color of file links
	urlcolor=cyan,           % color of external links
	linkbordercolor={1 1 1},
	citebordercolor={0 1 0},
	urlbordercolor={0 1 1}
}

\usepackage[hmarginratio=1:1, vmarginratio =1:1,
textheight=22cm, textwidth=18cm]{geometry}
\usepackage{verbatim}
\usepackage{authblk}
\usepackage{lineno}

\ifdraft{\linenumbers}{}

\makeatletter
\def\BState{\State\hskip-\ALG@thistlm}
\makeatother

\newcommand{\R}{\mathbb{R}}
\newcommand{\thresh}{\left\lceil \frac{m^2}{2n}\right\rceil}

\newcommand{\cS}{\mathcal{S}}

\newcommand{\E}{\mathbb{E}}
\renewcommand{\P}{\mathbb{P}}

\newcommand{\M}{\mathcal{M}}

\newcommand{\cp}{12}

\newcommand{\rd}{\color{red}}
\newcommand{\bk}{\color{black}}
\newcommand{\diag}{\text{diag}}
\newcommand{\round}{\text{round}}

\def\pace{{\sf PACE}}
\def\wpace{{\sf w-PACE}}
\def\gale{{\sf GALE}}

\def\match{{\sf Match}}
\def\dgcluster{{\sf DGCluster}}
\def\cluster{{\sf Cluster}}
\def\merge{{\sf Merge}}
\def\naivecluster{{\sf NaiveCluster}}
\def\rpk{{\sf RPKmeans}}

\def\fff#1{&{{\pageref{#1}}}\cr}
\def\hfff#1{\label{#1}}

\newtheorem{prop}{Proposition}[section]
\theoremstyle{plain}
\newtheorem{defn}{Definition}[section]
\theoremstyle{plain}
\newtheorem{theorem}{Theorem}[section]
\theoremstyle{plain}
\newtheorem{lemma}{Lemma}[section]
\theoremstyle{plain}

\theoremstyle{plain}
\newtheorem{cor}{Corollary}[section]
\theoremstyle{plain}

\theoremstyle{plain}
\newtheorem{remark}{Remark}[section]
\theoremstyle{plain}

\title{\textbf{Two provably consistent divide and conquer clustering algorithms for large networks}}

\author[1]{Soumendu Sundar Mukherjee}
\author[2]{Purnamrita Sarkar}
\author[1]{Peter J. Bickel}
\affil[1]{Department of Statistics, University of California, Berkeley}
\affil[2]{Department of Statistics and Data Sciences, University of Texas, Austin}

\begin{document}
\maketitle
\begin{abstract}
	In this article, we advance divide-and-conquer strategies for solving the community detection problem in networks. We propose two algorithms which perform clustering on a number of small subgraphs and finally patches the results into a single clustering. The main advantage of these algorithms is that they bring down significantly the computational cost of traditional algorithms, including spectral clustering, semi-definite programs, modularity based methods, likelihood based methods etc., without losing on accuracy and even improving accuracy at times. These algorithms are also, by nature, parallelizable. Thus, exploiting the facts that most traditional algorithms are accurate and the corresponding optimization problems are much simpler in small problems, our divide-and-conquer methods provide an omnibus recipe for scaling traditional algorithms up to large networks. We prove consistency of these algorithms under various subgraph selection procedures and perform extensive simulations and real-data analysis to understand the advantages of the divide-and-conquer approach in various settings.
 \end{abstract}

\tableofcontents

\section{Glossary of notation}
\label{sec:glossary}
For the convenience of the reader, we collect here some of the more frequently used notations used in the paper, and provide a summarizing phrase for each, as well as the page number at which the notation first appears.

\bigskip
\def\qq{&}

\begin{center}
\halign{
#\quad\hfill&#\quad\hfill&\quad\hfill#\cr
%%%%%%%%%%%%%%
$T$ \qq number of subgraphs used \fff{pace}
%%%%%%%%%%%%%%
$S, S_l$ \qq a randomly chosen subgraph \fff{pace}
%%%%%%%%%%%%%%
$w_{l,i,j}$ \qq weights used in \wpace{} \fff{pace}
%%%%%%%%%%%%%%
$N_{ij}$ \qq the number of times $i$ and $j$ appear in all subgraphs $S_1, \ldots, S_T$ \fff{pace}
%%%%%%%%%%%%%%
$\tau$ \qq  threshold for $N_{ij}$, a tuning parameter in \pace{} \fff{pace}
%%%%%%%%%%%%%%
$\hat{C}$ \qq estimated clustering matrix returned by \pace{} \fff{pace}
%%%%%%%%%%%%%%
$\hat{C}_{proj}$ \qq randomly projected version of $\hat{C}$, to be used for recovering an estimate of $Z$ \fff{cproj}
%%%%%%%%%%%%%%
$M$ \qq confusion matrix between two clusterings, used by \match{} \fff{match}
%%%%%%%%%%%%%%
$\hat{Z}_\ell$\qq cluster membership matrix of subgraph $S_\ell$ before alignment \fff{gale_defn}
%%%%%%%%%%%%%%
$\hat{Z}^{(\ell)}$\qq cluster membership matrix of subgraph $S_\ell$ after alignment \fff{gale_defn}
%%%%%%%%%%%%%%
$n_k$\qq number of nodes from cluster $k$ \fff{cluster_size}
%%%%%%%%%%%%%%
$\pi_k$\qq $n_k/n$, i.e. proportion of nodes from cluster $k$ \fff{cluster_size}
%%%%%%%%%%%%%%
$\pi_{\min}$\qq $\min_k \pi_k$, similarly $\pi_{\max}$ \fff{cluster_size}
%%%%%%%%%%%%%%
$\delta_c(\sigma, \sigma')$ \qq normalized mismatch between two clusterings $\sigma$, $\sigma'$ ($\sigma,\sigma' : \{1, \ldots, n\} \rightarrow \{1, \ldots, K\}$) \fff{misclustering_def}
%%%%%%%%%%%%%%
$\delta(Z, Z')$ \qq normalized mismatch between two clusterings $Z$, $Z'$ ($Z, Z' \in \{0, 1\}^{n \times K}$)\fff{misclustering_def}
%%%%%%%%%%%%%%
$\tilde{\delta}(C, C')$ \qq normalized between two clustering matrices $C$, $C'$ \fff{misclustering_def}
%%%%%%%%%%%%%%
%$\Pi$ \qq a $K\times K$ permutation matrix \fff{gale_results}
%%%%%%%%%%%%%%
$\M(Z', Z)$ \qq $n\delta(Z, Z')$ \fff{gale_results}
%%%%%%%%%%%%%%
$y^{(\ell)}_i$\qq $1$ if $i\in S_\ell$ and 0 otherwise \fff{gale_results}
%%%%%%%%%%%%%%
$\boldsymbol{\pi}$\qq the vector $(\pi_1,\dots,\pi_K)$ \fff{cluster_vec}
%%%%%%%%%%%%%%
$n_k^{(S)}$\qq number of nodes from cluster $k$ in a subgraph $S$ \fff{cluster_size_subgraphs}
%%%%%%%%%%%%%%
$Y_{ab}$\qq $|S_a\cap S_b|$ \fff{intersection_size}
%%%%%%%%%%%%%%
}\end{center}

\section{Introduction}\label{sec:intro}
Community detection, also known as community extraction or network clustering, is a central problem in network inference. A wide variety of real world problems, ranging from finding protein protein complexes in gene networks~\citep{guruharsha2011protein} to studying the consequence of social cliques on adolescent behavioral patterns~\citep{hollingshead1949elmstown} depend on detecting and analyzing communities in networks. In most of these problems, one observes co-occurrences or interactions between pairs of entities, i.e. pairwise edges and possibly additional node attributes. The goal is to infer the hidden community memberships. 

Many of these real world networks are massive, and hence it is crucial to develop and analyze scalable algorithms for community detection. We will first talk about methodology that uses only the network connections for inference. These can be divided into roughly two types. The first type consists of  methods which are derived independently of any model assumption. These typically involve the formulation of global optimization problem; examples include normalized cuts~\citep{shi2000normalized}, Spectral Clustering~\citep{ng2002spectral}, etc. 

On the other end, Statisticians often devise techniques under model assumptions. The simplest statistical model for networks with communities is the Stochastic Blockmodel (SBM)~\citep{holland1983stochastic}. The key idea in a SBM is to enforce stochastic equivalence, i.e. two nodes in the same latent community have identical probabilities of connection to all nodes in the network. There are many extensions of SBM. The Degree Corrected Stochastic Blockmodel~\citep{karrer2011stochastic} allow one to model varied degrees in the same community, whereas a standard SBM does not. Mixed membership blockmodels~\citep{airoldi2009mixed} allow a node to belong to multiple communities, whereas in a SBM, a node can belong to exactly one cluster. 

For an SBM generating a network with $n$ nodes and $K$ communities, one has a hidden community/cluster membership matrix $Z\in \{0,1\}^{n\times K}$, where $Z_{ik} = 1$ if node $i$ is in community $k \in [K]$. Given these memberships, the link formation probabilities are given as
\[
	\P(A_{ij} = 1 \mid Z_{ik} = 1, Z_{jk'} = 1) = B_{kk'},
\]
where $B$ is a $K\times K$ symmetric parameter matrix of probabilities. The elements of $B$ may decay to zero as $n$ grows to infinity, to model sparse networks. 

Typically the goal is to estimate the latent memberships consistently. A method outputting an estimate $\hat{Z}$ is called strongly consistent if $\P(\hat{Z}=Z\Pi)\rightarrow 1$ for some $K\times K$ permutation matrix $\Pi$, as $n\rightarrow\infty$. A weaker notion of consistency is when the fraction of misclustered nodes goes to zero as $n$ goes to infinity. Typically most of the consistency results are derived 
where average degree of the network grows faster than the logarithm of $n$. This is is often called the semi-dense regime. When average degree is bounded, we are in the sparse regime. %While, in the semi-dense regime, many methods have been shown to be consistent, in the sparse regime, one can only hope for performing better than a random predictor.

There are a plethora of algorithms for community detection. These include likelihood-based methods~\citep{amini2013pseudo}, modularity based methods~\citep{snijers1997mcmc,newman2004finding,bickel2009nonparametric}, spectral methods~\citep{rohe2011spec}, semi-definite programming (SDP) based approaches~\citep{cai2015robust} etc. Among these, spectral methods are scalable since the main bottleneck is computing top $K$ eigenvectors of a large and often sparse matrix. While the theoretical guarantees of Spectral Clustering are typically proven in the semi-dense regime~\citep{mcsherry2001spectral,rohe2011spec,lei2015consistency}, a regularized version of it has been shown to perform better than a random predictor for sparse networks~\citep{le2015sparse}. Profile likelihood methods~\citep{bickel2009nonparametric} involve greedy search over all possible membership matrices, which makes them computationally expensive. Semidefinite programs are robust to outliers~\citep{cai2015robust} and are shown to be strongly consistent in the dense regime~\citep{amini2014semidefinite} and yield a small but non-vanishing error in the sparse regime~\citep{guedon2014community}. However, semi
definite programs are slow and typically only scale to thousands of nodes, not millions of nodes.

Methods like spectral clustering on geodesic distances~\citep{bhattacharyya2014community} are provably consistent in the semi-dense case, and can give a small error in sparse cases. However, it requires computing all pairs of shortest paths between all nodes, which can pose serious problems for both computation and storage for very large graphs.

Monte Carlo methods~\citep{snijers1997mcmc,nowicki2001block}, which are popular tools in Bayesian frameworks, are typically not scalable. More scalable alternatives such as variational methods~\citep{gopalan2013efficient} do not have provable guarantees for consistency, and often suffer from bad local minima.

So far we have discussed community detection methods which only look at the network connections, and not node attributes, which are often also available and may possess useful information on the community structure (see, e.g., \cite{newman2015structure}). There are extensions of the methods mentioned earlier which accommodate node attributes, e.g., modularity based~\citep{zhang2015community}, spectral~\citep{binkiewicz2014covariate}, SDP based~\citep{yan2016covariate}, etc. These methods come with theoretical guarantees and have good performance in moderately sized networks. While existing Bayesian methods~\citep{mcauley2012learning,newman2015structure,xu2012model} are more amenable to incorporating covariates in the inference procedure, they often are computationally expensive and lack rigorous theoretical guarantees. 

While the above mentioned array of algorithms are diverse and each has unique aspects, in order to scale them to very large datasets, one has to apply different computational tools tailored to different algorithmic settings. While stochastic variational updates may be suitable to scale Bayesian methods, pseudo likelihood methods are better optimized using row sums of edges inside different node blocks.
 
In this article, we propose a divide and conquer approach to community detection. The idea is to apply a community detection method on small subgraphs of a large graph, and somehow stitch the results together. If we could achieve this, we would be able to scale up any community detection method (which may involve covariates as well as the network structure) that is computationally feasible on small graphs, but is difficult to execute on very large networks. This would be especially useful for computationally expensive community detection methods (such as SDPs, modularity based methods, Bayesian methods etc.). Another possible advantage concerns variational likelihood methods (such as mean-field) with a large number (depending on $n$) of local parameters, which typically have an optimization landscape riddled with local minima. For smaller graphs there are less parameters to fit, and the optimization problem often becomes easier.

Clearly, the principal difficulty in doing this, is matching the possibly conflicting label assignments from different subgraphs (see Figure~\ref{FiG:conflict}(a) for an example). This immediately rules out a simple-minded averaging of estimates $\hat{Z}_S$ of cluster membership matrices $Z_S$, for various subgraphs $S$, as a viable stitching method.
\begin{figure}[!htbp]
\centering
$\begin{array}{cc}
\includegraphics[width=0.5\textwidth]{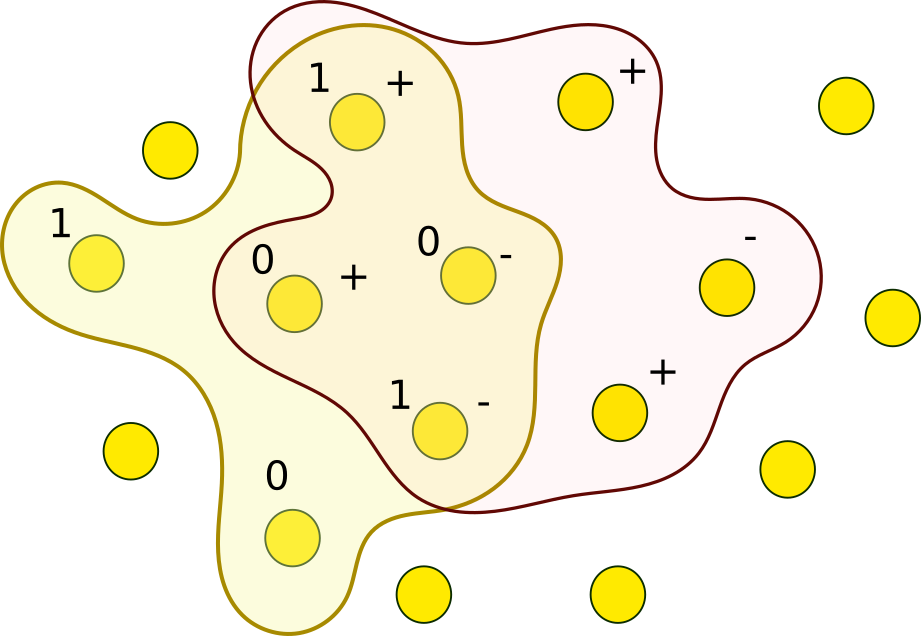}& \hskip30pt
\includegraphics[width = .4\textwidth]{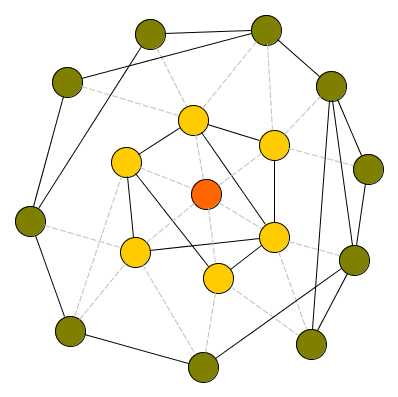}\\
(\text{a}) &  \hskip30pt (\text{b})
\end{array}$
\caption{(a) Conflicting label assignments on two subgraphs, on one labels are denoted by $0/1$, on the other by $+/-$, (b) onion neighborhood: the yellow vertices and the solid edges between them constitute the ego network of the root vertex colored red. The green and yellow vertices together with the solid edges between them constitute the $2$-hop onion neighborhood of the root}
\label{FiG:conflict}
\end{figure}

In this regard, we propose two different stitching algorithms.  The first is called Piecewise Averaged Community Estimation (\pace{}); in which we focus on estimating the clustering matrix $C=ZZ^\top$, which is labeling invariant, since the $(i,j)$-th element of this matrix being one simply means that nodes $i$ and $j$ belong to the same cluster, whereas the value zero means $i$ and $j$ belongs to two different clusters. Thus we first compute estimates of $Z_S Z_S^\top$ for various subgraphs $S$ and then average over these matrices to obtain an estimate $\hat{C}$ of $C$. Finally we apply some computationally cheap clustering algorithm like approximate $K$-means, \dgcluster{}\footnote[2]{This is a greedy algorithm we propose in this article, see Appendix~\ref{sec:app_greedy}.},  spectral clustering etc. on $\hat{C}$ to recover an estimate of $Z$.

We also propose another algorithm called Global Alignment of Local Estimates (\gale{}), where we first take a sequence of subgraphs, such that any two consecutive subgraphs on this sequence have a large intersection, and then traverse through this sequence, aligning the clustering based on a subgraph with an averaged clustering of the union of all its predecessor subgraphs in the sequence, which have already been aligned. The alignment is done via an algorithm called \match{} which identifies the right permutation to align two clusterings on two subgraphs by computing the confusion matrix of these two clusterings restricted to the intersection of the two subgraphs. Whereas a naive approach would entail searching through all $K\!$ permutations, \match{} finds the right permutation in $K\log K$ time. Once this alignment step is complete, we get an averaged clustering of the union of all subgraphs (which covers all the vertices). By design \gale{} works with estimates of cluster membership matrices $Z_S$ directly to output an estimate of $Z$, and thus, unlike \pace{}, avoids the extra overhead of recovering such an estimate from $\hat{C}$.   

The rest of the paper is organized as follows. In Section~\ref{sec:local_algo} we describe our algorithms. In Section~\ref{sec:main_results} we state our main results and some applications. Section~\ref{sec:simulations} contains simulations and real data experiments. In Section~\ref{sec:proofs} we provide proofs of our main results, while relegating some of the details to the Appendix, Section~\ref{sec:app_aux_results}. Finally, in Section~\ref{sec:discussions} we conclude with some discussions and directions for future work. 

%%%%%%%%%%%%%%%%%%%%%%%%%%%%%%%%%%%%%%%%%%%%%%%%%%%%%%%%
\section{Two divide and conquer algorithms}\label{sec:local_algo}
As we discussed in the introduction, the main issue with divide and conquer algorithms for clustering is that one has to somehow match up various potentially conflicting label assignments. In this vein we propose two algorithms to accomplish this task. Both algorithms first compute the clustering on small patches of a network; these patches can be induced subgraph of a random subsample of nodes, or neighborhoods. However, the stitching procedures are different.

%%%%%%%%%%%%%%%%%%%%%%%%%%%%%%%%%%%%%%%%%%%%%%%%%%%%%%%%
\subsection{\pace{}: an averaging algorithm}
Suppose $A$ is the adjacency matrix of a network with true cluster membership of its nodes being given by the $n\times K$ matrix $Z$ where there are $K$ clusters. Set $C = ZZ^\top$ to be the clustering matrix whose $(i,j)$-th entry is the indicator of whether nodes $i,j$ belong to the same cluster. Given $A$ we will perform a local clustering algorithm to obtain an estimate of $C$, from which an estimate $\hat{Z}$ of the cluster memberships may be reconstructed.
\begin{algorithm}[H]
		\caption{\pace{}: Piecewise Averaged Community Estimation \hfff{pace}}
	\label{mainalg}
	\begin{algorithmic}[1]
		\State {\bf Subgraph selection:} Fix a positive integer threshold $m_{\star}$ for minimum required subgraph size. Fix another positive integer $T$, that will be the number of subgraphs we will sample. Given $A$,  choose $T$ subsets $S_1,\ldots,S_T$ of the nodes with $|S_{\ell}| \ge m_{\star}$ by some procedure, e.g., select $m \ge m_{\star}$ nodes at random, or pick $h$-hop neighborhoods of vertices in $G$, or ego-neighborhood of vertices. By $A_{S_{\ell}}$ denote the adjacency matrix of the network induced by $S_{\ell}$.\vskip5pt
		\State {\bf Clustering on subgraphs:}
		Perform any of the standard clustering algorithms like Profile Likelihood (PL), Mean Field Likelihood (MFL), Spectral Clustering (SC), Semi-Definite Programming (SDP) etc. on each of these $T$ subgraphs which have size at least $m_{\star}$ to obtain estimated clustering matrices $\hat{C}^{(S_{\ell})} = \hat{C}^{(\ell)}$. For the rest of the subgraphs set $\hat{C}^{(\ell)} \equiv 0$.  Extend $\hat{C}^{(\ell)}$ to an $n \times n$ matrix by setting $\hat{C}^{(\ell)}_{ij} = 0$ if at least one of $i,j$ was not selected in $S_{\ell}$. Let's denote the resulting matrix again by $\hat{C}^{(\ell)}$.\vskip5pt  
		\State {\bf Patching up:} Let $y_{ij}^{(\ell)}$ denote the indicator of the event that both $i,j$ were selected in $S_{\ell}$. Set $N_{ij} = \sum_{\ell=1}^T y_{ij}^{(\ell)}$. Define the combined estimator $\hat{C} = \hat{C}_{\tau}$ by
		\begin{equation}\label{eq:pace}
		\hat{C}_{ij} = \hat{C}_{\tau, ij} = \frac{\mathbf{1}_{\{N_{ij} \ge \tau \}}\sum_{\ell = 1}^T y_{ij}^{(\ell)} \hat{C}^{(\ell)}_{ij}}{N_{ij}} = \frac{\mathbf{1}_{\{N_{ij} \ge \tau \}}\sum_{\ell = 1}^T\hat{C}^{(\ell)}_{ij}}{N_{ij}}.
		\end{equation}
		Here $1 \le \tau \le T$ is an integer tuning parameter.
		We will call $\hat{C}_{\tau}$ as Piecewise Averaged Community Estimator (also abbreviated as \pace{}).
	\end{algorithmic}
\end{algorithm}

The $\tau$ parameter in \pace{} seems to reduce variance in estimation quality as it discards information from less credible sources --- if a pair of nodes has appeared in only a few subgraphs, we do not trust what the patching has to say about them. Setting $\tau = \theta\E(N_{ij})$ seems to work well in practice (this choice is also justified by our theory).

A slight variant of Algorithm~\ref{mainalg} is where we allow subgraph and/or node-pair specific weights $w_{\ell, i, j}$ in the computation of the final estimate, i.e.
\begin{equation}\label{eq:wpace}
\hat{C}_{ij} = \hat{C}_{\tau, ij} = \frac{\mathbf{1}_{\{N_{ij} \ge \tau \}}\sum_{\ell = 1}^T w_{\ell, i ,j} y_{ij}^{(\ell)} \hat{C}^{(\ell)}_{ij}}{N_{ij}} = \frac{\mathbf{1}_{\{N_{ij} \ge \tau \}}\sum_{\ell = 1}^T w_{\ell, i, j}\hat{C}^{(\ell)}_{ij}}{N_{ij}},
\end{equation}
where $N_{ij}$ now equals $\sum_{\ell=1}^T w_{\ell, i, j} y_{ij}^{(\ell)}$. We may call this estimator \wpace{} standing for weighted-\pace{}. If the weights are all equal, \wpace{} becomes equivalent to ordinary \pace{}. There are natural nontrivial choices, including
\begin{itemize}
	\item[(i)] $w_{\ell, i, j} = |S_{\ell}|,$ which will place more weight to estimates based on large subgraphs,

	\item[(ii)] $w_{\ell, i, j} = \text{deg}_{S_{\ell}}(i) + \text{deg}_{S_{\ell}}(j)$, where $\text{deg}_{S}(u)$ denotes the degree of node $u$ in  subgraph $S$ (this will put more weight on pairs which have high degree in $S_{\ell}$).
\end{itemize}
The first prescription above is intimately related to the following sampling scheme for ordinary \pace{}: pick subgraphs with probability proportional to their sizes. For instance, in Section~\ref{sec:realdata} we analyze the political blog data of \cite{adamic2005political} where neighborhood subgraphs are chosen by selecting their roots with probability proportional to degree.

In real world applications, it might make more sense to choose these weights based on domain knowledge (for instance, it may be that certain subnetworks are known to be important). Another (minor) advantage of having weights is that when $T = 1$ and $|S_1| = n$, we have $N_{ij} = w_{1, i, j}$ and so if $w_{1, i, j} \ge \tau$, then
\[
\hat{C}_{ij} = \hat{C}_{ij}^{(1)},
\]
i.e. \wpace{} becomes the estimator based on the full graph. This is for example true with $w_{\ell, i, j} = |S_{\ell}|$, because $\tau$ is typically much smaller than $n$. However, ordinary \pace{} lacks this property unless $\tau = 1$, in fact, with $\tau > 1$, the estimate returned by \pace{} is identically $0$.  Anyway, in what follows, we will stick with ordinary \pace{} because of its simplicity.

Before we discuss how to reconstruct an estimate $\hat{Z}$ of $Z$ from $\hat{C}$, let us note that we may obtain a binary matrix $\hat{C}_{\eta}$ by thresholding $\hat{C}$ at some level $\eta$ (for example, $\eta = 1/2$):
\[
	\hat{C}_{\eta} := [\hat{C} > \eta].
\]
This thresholding does not change consistency properties (see Lemma~\ref{lem:consistency_thresholded_mx}). Looking at a plot of this matrix gives a good visual representation of the community structure. In what follows, we work with unthresholded $\hat{C}$.

\noindent\textbf{Reconstruction of $\hat{Z}$:} How do we actually reconstruct $\hat{Z}$ from $\hat{C}$? The key is to note that members of the same community have identical rows in $C$ and that, thanks to \pace{}, we have gotten hold of a consistent estimate of $C$. Thus we may use any clustering algorithm on the rows of $\hat{C}$ to recover the community memberships. Another option would be to run spectral clustering on the matrix $\hat{C}$ itself. However, as the rows of $\hat{C}$ are $n$-vectors, most clustering algorithms will typically have a running time of $O(n^3)$\footnote[2]{A word on asymptotic notation: in addition to the standard `$O$', `$\Omega$', `$\Theta$', `$o$' notations, and their probabilistic counterparts, we will use the following less standard alternatives to clean up the presentation at places: 
%(i) $A \preceq B$ to mean $A = O(B)$, (ii) $A \succeq B$ to mean $A = \Omega(B)$,
(i) $A \asymp B$ to mean $A = \Theta(B)$, and (ii) $A \gg B$ to mean $B = o(A)$. We will also sometimes use a `tilde' over the standard notations to hide polylogarithmic factors.} at best. Indeed, the main computational bottleneck of any distance based clustering algorithm in a high dimensional situation like the present one, is computing $d_{ij} = \|\hat{C}_{i\star} - \hat{C}_{j\star}\|$ which takes $O(n)$ bit operations. However, since we have gotten a good estimate of $C$, we can project the rows of $\hat{C}$ onto lower dimensions, without distorting the distances too much. The famous Johnson-Lindenstrauss Lemma for random projections says that by projecting onto $\Omega(\log n/\epsilon^2)$ dimensions, one can keep, with probability at least $1 - O(1/n)$, the distances between projected vectors within a factor of $(1 \pm \epsilon)$ of the true distances. Choosing $\epsilon$ as inverse polylog$(n)$ we need to project onto polylog$(n)$ dimensions and this would then readily bring the computational cost of any distance based algorithm down from $O(n^3)$ to $O(n^2 \text{polylog}(n))$. 

Following the discussion of the above paragraph, we first do a random projection of the rows of $\hat{C}$ onto $s\, (= \text{polylog}(n))$ dimensions and then apply a (distance-based) clustering algorithm.
\hfff{cproj}
\begin{algorithm}[H]
	\caption{Recovering $\hat{Z}$ from $\hat{C}$: random projection followed by distance based clustering}
	\label{alg:rec_Z_rproj}
	\begin{algorithmic}[1]
		\State Select a dimension $s$ for random projection. Let \cluster{}$(\cdot, K)$ be a clustering algorithm that operates on the rows of its first argument and outputs $K$ clusters.
		\State $\hat{C}_{proj} \gets \hat{C}R/\sqrt{s}$, where $R$ is a standard Gaussian matrix of dimensions $n\times s$. \vskip5pt
		\State $\hat{Z} \gets \cluster{}(\hat{C}_{proj}, K)$.
	\end{algorithmic}
\end{algorithm}

As for $\cluster{}(\cdot,K)$, we may use approximate $K$-means or any other distance based clustering algorithm, e.g., \dgcluster{}$(\hat{C}, \hat{C}_{proj}, K)$, a greedy algorithm presented in Appendix~\ref{sec:app_greedy} as Algorithm~\ref{alg:dgcluster}.

\subsection{\gale{}: a sequential algorithm}
First we introduce a simple algorithm for computing the best permutation to align labels of one clustering ($Z_1$) to another ($Z_2$) of the same set of nodes (with fixed ordering) in a set $S$. The idea is to first compute the confusion matrix between two clusterings. Note that if the two labelings each have low error with respect to some unknown true labeling, then the confusion matrix will be close to a diagonal matrix up to permutations. The following algorithm below essentially finds a best permutation to align one clustering to another. 
\hfff{match}
\begin{algorithm}[H]
	\caption{\match{}: An algorithm for aligning two clusterings of the same set $S$ of nodes. Input $Z_1,Z_2 \in \{0,1\}^{|S|\times K}$.}
	\label{alg:match}
	\begin{algorithmic}[1]
		\State Compute $K\times K$ confusion matrix $M = Z_1^\top Z_2$. Set $\Pi \gets 0_{K\times K}$.
		\While {there are no rows/columns left}
			\begin{itemize}
			\item[(a)] Find $i,j$, such that $M_{ij}=\|M\|_{\infty}$ (a tie can be broken arbitrarily).
			\item[(b)] Set $\Pi_{ij} = 1$.
			\item[(c)] Replace the $i$-{th} row and $j$-th columns in $M$ with $-1$.
			\end{itemize}
		\EndWhile
		\State Return the permutation matrix $\Pi$.	
	\end{algorithmic}
\end{algorithm}
\begin{remark}
	One can also compute the optimal permutation by searching through all $K!$ permutations of the labels and picking the one which gives smallest mismatch between the two; but \match{} brings the dependence on $K$ down from exponential to quadratic. Note that if one of the clusterings are poor, then \match{} may not retrieve the optimal permutation. However, our goal is to cluster many subgraphs using an algorithm which has good accuracy with high probability (but may be computationally intensive, e.g., Profile Likelihood or Semidefinite Programming) and then use the intersections between the subgraphs to align one to another. As we shall show later, as long as there are enough members from each community in $S$, the simple algorithm sketched above suffices to find an optimal permutation. %in this setting. 
	%We also want to point out that when there are fewer labels in one clustering, then that row or column will be all zeros. In our analysis we show that this happens with vanishing probability.
\end{remark}
Now we present our sequential algorithm which aligns labelings across different subgraphs. The idea is to first fix the indexing of the nodes; cluster the subgraphs (possibly with a parallel implementation) using some algorithm, and then align the clusterings along a sequence of subgraphs. To make things precise, we make the following definition.
\begin{defn}
	Let $\cS_{m,T}=(V,E)$ (vertex set $V$ and edge set $E$) denote a ``super-graph of subgraphs'' where each subgraph is a node, and two nodes are connected if the corresponding subgraphs have a substantial overlap. For random $m$-subgraphs, this threshold is $m_1 := \thresh$.
	
	Define a traversal through a spanning tree $\mathcal{T}$ of $\cS_{m,T}$ as a sequence $x_1, \dots, x_{J}$, $x_j \in [T]$, $T \le J \le 2T$, covering all the vertices, such that along the sequence $x_i$ and $x_{i+1}$ are adjacent, for $1 \le i \le J-1$ (i.e. a traversal is just a walk on $\mathcal{T}$ of length at most $2T - 1$ passing through each vertex at least once). 
	%If the spanning tree is a path, then $J=T$, whereas for a spanning tree thats a star graph, $J=2T-1$. For any two consecutive vertices $i_k,i_{k+1}$ on the traversal, either $(i_k,i_{k+1})\in E$ and $i_{k+1}$ has not appeared in the traversal before, or $i_{k+1}$ has been encountered on the traversal already.
\end{defn}

 After we construct a traversal, we travel through this traversal such that at any step, we align the current subgraph's labels using the \match{} algorithm (Algorithm~\ref{alg:match}) on its intersection with the union of the previously aligned subgraphs. At the end, all subgraph labellings are aligned to the labeling of the starting subgraph. Now we can simply take an average or a majority vote between these.
\begin{algorithm}[H]
	\caption{\gale{}: Global Alignment of Local Estimates. Input adjacency matrix $A$, parameters $m,T, \tau_i, i \in [T]$, a base algorithm $\mathcal{A}$ (e.g., PL, MFL, SC, SDP etc. )}
	\label{alg:seq}
	\begin{algorithmic}[1]
		\State {\bf Subgraph selection:} Given $A$,  choose $T$ subsets $S_1,\ldots,S_T$ of the nodes by some procedure, e.g., select them at random, or select $T$ random nodes and then pick their $h$-hop neighborhoods. By $A_{S_{\ell}}$ denote the adjacency matrix of the network induced by $S_{\ell}$.\vskip5pt
		\State {\bf Clustering on subgraphs:}
		Perform algorithm $\mathcal{A}$ on each of these $T$ subgraphs to obtain estimated cluster membership matrices \hfff{gale_defn} $\hat{Z}_{\ell} = \hat{Z}_{S_\ell}$. Extend $\hat{Z}_{\ell}$ to a $n \times K$ matrix by setting $(\hat{Z}_{\ell})_{jk} = 0$ for all $k \in [K]$ if $j\notin S_\ell$. \vskip5pt 
		
		\State {\bf Traversal of subgraphs: } Construct a traversal $S_{x_1},\dots,S_{x_J}$ through the $T$ subgraphs. 

		\State {\bf Initial estimate of $Z$: } $\hat{Z} \gets \hat{Z}_{x_1}$. Also set $\hat{Z}^{(x_1)} := \hat{Z}_{x_1}$.
		
		\State {\bf Sequential label aligning: } For subgraph $S_{x_i}$ on the traversal ($i = 1, \ldots, J$), if $x_i$ has not been visited yet,
		\begin{itemize}
			\item[(a)] Compute the overlap between the current subgraph with all subgraphs previously visited, i.e. let $S = S_{x_i} \cap \cup_{\ell=1}^{i-1}S_{x_\ell}$.

			\item[(b)] Compute the best permutation to match the clustering $\hat{Z}_{x_i}, \hat{Z}$ on this set $S$, i.e. compute $\hat{\Pi}_i = \match{}(\hat{Z}_{x_i}\big|_S,\hat{Z}\big|_S)$.

			\item[(c)] Permute the labels of the nodes of $S_{x_i}$ to get an aligned cluster membership matrix $\hat{Z}^{(x_i)}\gets \hat{Z}_{x_i}\hat{\Pi}$.

			\item[(d)] Update $\hat{Z}_{jk} \gets \frac{\sum_{\ell \in \{x_1, \ldots, x_i\}} \hat{Z}^{(x_\ell)}_{jk} \mathbf{1}_{\{j\in S_{x_\ell}\}}}{\sum_{\ell \in \{x_1, \ldots, x_i\}} \mathbf{1}_{\{j\in S_{x_\ell}\}}} \mathbf{1}_{\{\sum_{\ell \in \{x_1, \ldots, x_i\}} \mathbf{1}_{\{j\in S_{x_\ell}\}} > \tau_{i}\}},$ for some threshold $\tau_i.$ 

			\item[(e)] Mark $S_{x_i}$ as visited.
		\end{itemize}	 		
	\end{algorithmic}
\end{algorithm}

\noindent
\textbf{Implementation details:} Constructing a traversal of the subgraphs can be done using a depth first search of the super-graph $\cS_{m,T}$ of subgraphs. For our implementation, we start with a large enough subgraph (the parent), pick another subgraph that has a large overlap with it (the child), align it and note that this subgraph has been visited. Now recursively find another unvisited child of the current subgraph, and so on. It is possible that a particular path did not cover all vertices, and hence it is ideal to estimate clusterings with multiple traversals with different starting subgraphs and then align all these clusterings, and take an average. This is what we do for real networks. We also note that at any step, if we find a poorly clustered subgraph, then this can give a bad permutation which may deteriorate the quality of aligning the subsequent subgraphs on the path. In order to avoid this we use a self validation routine. Let $S$ be intersection of current subgraph with union of the previously visited subgraphs. After aligning the current subgraph's clustering, we compute the classification accuracy of the current labeling of $S$ with the previous labeling of $S$. If this accuracy is large enough, we use this subgraph, and if not we move to the next subgraph on the path. For implementation, we use a threshold of $0.55$.

\vskip10pt
\noindent \textbf{Computational time and storage:}
The main computational bottleneck in \gale{} is in building a traversal through the random subgraphs. Let $\eta_{m,T}$ be the time for computing the clusterings for $T$ subgraphs in parallel. A naive implementation would require computing intersections between all $\binom{T}{2}$ pairs of $m$-subsets. As we will show in our theoretical analysis, we take $m=\omega(\sqrt{n/\pi_{\min}})$, \hfff{cluster_size} where $\pi_{\min}:=\min_k \pi_k$ (here $\pi_k := n_k/n$, where $n_k$ is the size of the $k$-th cluster) and $T=\tilde{\Omega}(n/m)$. Taking $\pi_{\min}=\Theta(1)$, the computation of intersections takes $O(T^2 m)=\tilde{O}(n^{3/2})$ time. Further, a naive comparison for computing  subsets similar or close to a given one would require $T\log T$ time for each subset leading to $T^2\log T=\tilde{O}(n)$ computation. However, for building a traversal one only needs access to subsets with large overlap with a given subset, which is a classic example of nearest neighbor search in Computer Science. 

One such method is the widely used and theoretically analyzed technique of Locality Sensitive Hashing (LSH). A hash function maps any data object of an arbitrary size to another object of a fixed size. In our case, we map the characteristic vector of a subset to a number. The idea of LSH is to compute hash functions for two subsets $A$ and $B$ such that the two functions are the same with high probability if $A$ and $B$ are ``close''.  In fact, the amount of overlap normalized by $m$ is simply the cosine similarity between the characteristic vectors $\chi(A)$ and $\chi(B)$ of the two subsets, for which efficient hashing schemes $h:\{0,1\}^n\rightarrow \{0,1\}$ using random projections exist~\cite{Charikar:2002}, with 
\[
\P(h(A)=h(B))=1-\arccos(\chi(A)^T\chi(B)/m)/\pi.
\]
For LSH schemes, one needs to build $L:=T^\rho$ hash tables, for some $\rho<1$, that governs the approximation quality. In each hash table, a ``bucket'' corresponding to an index stores all subsets which have been hashed to this index. For any query point, one evaluates $O(L)$ hash functions and examines $O(L)$ subsets hashed to those buckets in the respective hash tables. Now for these subsets, the distance is computed exactly. The preprocessing time is $O(T^{1+\rho})=\tilde{O}(n^{\frac{1+\rho}{2}})$, with storage being $O(T^{1+\rho}+Tm)=\tilde{O}(n)$, and total query time being $O(T^{1+\rho}m)=\tilde{O}(n^{1+\rho/2})$. This brings down the running time added to the algorithm specific $\eta_{m,T}$ from sub-quadratic time $\tilde{O}(n^{3/2})$ to nearly linear time, i.e. $\tilde{O}(n^{1+\rho/2})$.

Thus, for other nearly linear time clustering algorithms \gale{} may not lead to computational savings. However, for algorithms like Profile Likelihood or SDP which are well known to be computationally intensive, \gale{} can lead to a significant computational saving without requiring a lot of storage.

\subsection{Remarks on sampling schemes}
With \pace{} we have mainly used random $m$-subgraphs, $h$-hop neighborhoods, and onion neighborhoods, but many other subgraph sampling schemes are possible. For instance, choosing roots of hop neighborhoods with probability proportional to degree, or sampling roots from high degree nodes (we have done this in our analysis of the political blog data, in Section~\ref{sec:realdata}). As discussed earlier, this weighted sampling scheme is related to \wpace{}. A natural question regarding $h$-hop neighborhoods is how many hops to use. While we do not have a theory for this yet, because of ``small world phenomenon'' we expect not to need many hops; typically in moderately sparse networks, $2$-$3$ hops should be enough. Although, an adaptive procedure (e.g., cross-validation type) for choosing $h$ would be welcome. Also, since neighborhood size increases exponentially with hop size, an alternative to choosing full hop-neighborhoods is to choose a smaller hop-neighborhood and then add some (but not all) randomly chosen neighbors of the already chosen vertices. Other possibilities include sampling a certain proportion of edges at random, and the consider the subgraph induced by the participating nodes. We leave all these possibilities for future work. 

We have analyzed \gale{} under the random sampling scheme. For any other scheme, one will have to understand the behavior of the intersection of two samples or neighborhoods. For example, if one takes $h$-hop neighborhoods, for sparse graphs, each neighborhood predominantly has nodes from mainly one cluster. Hence \gale{} often suffers with this scheme. We show this empirically in Section~\ref{sec:simulations}, where \gale{}'s accuracy is much improved under a random $m$-subgraph sampling scheme.

\subsection{Beyond community detection}
The ideas behind \pace{} and \gale{} are not restricted to community detection and can be modified for application in other interesting problems, including general clustering problems, co-clustering problems (\cite{rohe2016co}), mixed membership models, among others (these will be discussed in an upcoming article). In fact, \citet{JMLR:v16:mackey15a} took a similar divide and conquer approach for matrix completion.

%%%%%%%%%%%%%%%%%%%%%%%%%%%%%%%%%%%%%%%%%%%%%%%%%%%%%%%%
\section{Main results}\label{sec:main_results}
In this section we will state and discuss our main results on \pace{} and \gale{}, along with a few applications.
\subsection{Results on \pace{}}\label{sec:pace_analysis}
Let $\sigma$ and $\sigma'$ be two clusterings (of $n$ objects into $K$ clusters), usually their discrepancy is measured by \hfff{misclustering_def}
\[
	\delta_c(\sigma,\sigma') = \inf_{\xi \in S_K} \frac{1}{n}\sum_{i=1}^n \mathbf{1}_{\{ \sigma(i) \ne \xi(\sigma'(i))\}} = \frac{1}{n} \inf_{\xi \in S_K} \|\xi(\sigma)  - \sigma'\|_0,
\]
where $S_K$ is the permutation group on $[K]$. If $Z,Z'$ are the corresponding $n\times K$ binary matrices, then a related measure of discrepancy between the two clusterings is $\delta(Z,Z') = \inf_{Q \text{ perm.}} \frac{1}{n}\|ZQ - Z'\|_0$. It is easy to see that $\delta(Z,Z') = 2 \delta_c(\sigma,\sigma')$. (To elaborate, let $Q_\xi$ be the permutation matrix corresponding to the permutation $\xi$, i.e. $Q_{ij} = \mathbf{1}_{\{ \xi(i) = j)\}}$. Then $\xi(\sigma(i)  \ne \sigma'(i)$, if and only if $(ZQ_{\xi})_{i\star} \ne Z'_{i\star},$ i.e. $\|(ZQ_{\xi})_{i\star} - Z'_{i\star}\|_0 = 2$.) For our purposes, however, a more useful measure of discrepancy would be the normalized Frobenius squared distance between the corresponding clustering matrices $C = ZZ^\top$ and $C' = Z'Z'^\top$, i.e.
\[
	\tilde{\delta}(C,C') = \frac{1}{n^2} \|C - C'\|_F^2.
\]
Now we compare these two notions of discrepancies.
\begin{prop}\label{prop:comparison_ineq}
We have $\tilde{\delta}(C,C') \le 4\delta(Z,Z') = 8 \delta_c(\sigma,\sigma')$
\end{prop}
Incidentally, if the cluster sizes are equal, i.e. $n/K$, then one can show that
\[
		\tilde{\delta}(C,C') \le \frac{4}{K}\delta(Z,Z') = \frac{8}{K}\delta_c(\sigma,\sigma').
	\]
Although we do not have a lower bound on $\tilde{\delta}(C,C')$ in terms of $\delta(Z,Z')$, Lemma A.1 of \cite{tang2013universally} gives us (with $X = Z$, $Y = Z'$) that there exists an orthogonal matrix $\mathcal{O}$ such that
\[
\|Z\mathcal{O} - Z'\|_F \le \frac{\|C - C'\| (\sqrt{K \|C\|} + \sqrt{K\|C'\|)}}{\lambda_{\min}(C)} \le \frac{2\sqrt{Kn }\|C - C'\|}{n_{\min}(Z)},
\]
where we used the fact that $\|C\| = \lambda_{\max}(C) = n_{\max}(Z) \le n.$ The caveat here is that the matrix $\mathcal{O}$ need not be a permutation matrix.

To prove consistency of \pace{} we have to assume that the clustering algorithm $\mathcal{A}$ we use has some consistency properties. For example, it will suffice to assume that for a randomly chosen subgraph $S$ (under our subgraph selection procedure), $\E \delta(\hat{Z}_S, Z_S)$\footnote[2]{The expectation is taken over both the randomness in the graph and the randomness of the sampling mechanism.} is small. The following is our main result on the expected misclustering rate of \pace{}.

\begin{theorem}[Expected misclustering rate of \pace{}]\label{thm:pace-main}
Let $S$ be a randomly chosen subgraph according to our sampling scheme. Let $\pi_{\max} = \max_k \pi_k$. We have
\begin{equation}\label{bound-pace}
	\E \tilde{\delta}(\hat{C},C) \le \frac{T}{\tau n^2} \times \E\|\hat{C}^{(S)} - C^{(S)}\|_F^2 + \pi_{\max} \times \max_{i,j} \P(N_{ij} < \tau),
\end{equation}
where
\[
\E\|\hat{C}^{(S)} - C^{(S)}\|_F^2 \le n^2\P(|S| < m_{\star}) +  4\E |S|^2 \delta(\hat{Z}_{S}, Z_{S}) \mathbf{1}_{(|S| \ge m_{\star})}.
\] 
\end{theorem}

The first term in (\ref{bound-pace}) essentially measures the performance of the clustering algorithm we use on a  randomly chosen subgraph. The second term measures how well we have covered the full graph by the chosen subgraphs, and only depend on the subgraph selection procedure. The effect of the algorithm we use is felt through the first term only. 

We can now specialize Theorem~\ref{thm:pace-main} to various subgraph selection schemes. First, we consider randomly chosen $m$-subgraphs, which is an easy corollary.

\begin{cor}[Subgraphs are induced by $m \ge m_{\star}$ randomly chosen nodes]\label{cor:m-subgraph} Let $p = \frac{m(m-1)}{n(n-1)}$, $0 < \theta < 1$ and $\tau = \theta Tp$. We have
\begin{equation}\label{bound1:random}
	\E \tilde{\delta}(\hat{C},C) \le \frac{m}{\theta (m - 1)} \times \E\tilde{\delta}(\hat{C}^{(S)}, C^{(S)}) + \pi_{\max} \times e^{ - (1 - \theta)^2 Tp / 2}.
\end{equation}
\end{cor}
Notice that the constant $\frac{m}{\theta(m -1)}$ in \eqref{bound1:random} can be made as close to $1$ as one desires, which means that the above bound is essentially optimal.  
% The takeaway from this bound is that typically, when $d_n \gg \log n$ and $m$ is large enough, the first term will go to zero, and only the second term, which measures how well we cover the original graph, will have to be made small by appropriately selecting $T$ and $m$. 

Full $h$-hop neighborhood subgraphs are much harder to analyze and will not be pursued here. However, ego networks, which are $1$-hop neighborhoods minus the root node (see Figure~\ref{FiG:conflict}(b)), are easy to deal with. One can also extend our analysis to $h$-hop onion neighborhoods which are recursively defined as follows: $\mathcal{O}_1(v) = S_1(v)$ is just the ego network of vertex $v$; in general, the $h$-th shell $S_h(v) := \uplus_{u\in S_{h-1}(v)} [\mathcal{O}_1(u)\setminus \mathcal{O}_{h - 1}(v) \cup \{v\}]$, and $\mathcal{O}_h(v) = \mathcal{O}_{h-1}(v) \uplus S_h(v)$, where the operation $\uplus$ denotes superposition of networks. Here, for ease of exposition, we choose to work with ego networks ($1$-hop onion neighborhoods).

\begin{cor}[Ego neighborhoods under a stochastic block model]\label{cor:ego-subgraph}
Let $B_\# = \max B_{ab}, B_{\star} = \min B_{ab}$ and $\tau = \frac{TB_{\star}^2}{4}$ and $m_{\star} \le (n - 1)B_{\star}/2$. Let $\theta > 1$. Then\footnote[2]{Actually we can allow $\tau = \theta' T B_{\star}^2$, for any $0 < \theta' < 1$. As the resulting bound involves complicated constants depending on $\theta'$ and does not add anything extra as to the nature of the bounds, we have chosen to work with a particular $\theta'$ $(= 1/4)$ to ease our exposition. With a general $\theta'$, the constant multiplier in the first term in \eqref{bound2:egonet} will be $4\frac{(1 + \theta)^2}{\theta'}$, instead of $16(1+\theta)^2$.}
\begin{equation}\label{bound2:egonet}
	\E \tilde{\delta}(\hat{C},C) \le \frac{16(1 + \theta)^2 B_{\#}^2}{B_{\star}^2}  \times \E \delta(\hat{Z}_{S}, Z_{S}) \mathbf{1}_{(|S| \ge m_{\star})} + \Delta,
\end{equation}
where
\[
\Delta  \le \frac{4}{B_{\star}^2} \times \exp\left( - \frac{nB_{\star}^2}{16B_{\#}}\right) + \frac{4}{n^2 B_{\star}^2} \times \exp \left( - \frac{\theta^2 n B_{\#}}{6} \right)
	+ \pi_{\max} \times \left[2\exp\left(-\frac{nB_{\star}^4}{16B_{\#}^2}\right) + \exp\left(- \frac{T B_{\star}^2}{16}\right)\right].
\]
\end{cor}

We will now use existing consistency results on several clustering algorithms $\mathcal{A}$, in conjunction with the above bounds to see what conditions (i.e. conditions on the model parameters, and $m, T$ etc.) are required for \pace{} to be consistent. We first consider $(1+\epsilon)$-approximate adjacency spectral clustering (ASP) of \cite{lei2015consistency} as $\mathcal{A}$. We will use stochastic block model as the generative model and for simplicity will assume that the link probability matrix has the following simple form
\begin{equation}\label{eq:simple_sbm}
	B = \alpha_n [ \lambda I + (1-\lambda) \mathbf{1}\mathbf{1}^\top] \text{ with $0 < \lambda < 1$.}
\end{equation}
We now quote a slightly modified version of Corollary~3.2 of \cite{lei2015consistency} for this model.
\begin{lemma}[\cite{lei2015consistency}]\label{lem:lei-rinaldo}
	Let $c_0, \epsilon, r > 0$. Consider an adjacency matrix $A$ generated from the simple block model \eqref{eq:simple_sbm} where $\alpha_n \ge c_0\log n/n$. If $\hat{Z}$ is the output of the $(1+\epsilon)$-approximate adjacency spectral clustering algorithm applied to $A$, then there exists an absolute constant $c=c(c_0,r)$ such that with probability at least $1-n^{-r}$,
	\begin{equation}\label{eq:lei_rinaldo}
		\frac{1}{2}\delta(\hat{Z},Z) \le \min\left\{ c^{-1}(2+\epsilon)\frac{K}{\lambda^2\alpha_n}\frac{\pi_{\max}}{\pi_{\min}^2 n}, 1\right\}.
	\end{equation}
\end{lemma}

\begin{cor}[$(1+\epsilon)$-approximate adjacency spectral clustering with random $m$-subgraphs]\label{cor:asp} Assume the setting of Lemma~\ref{lem:lei-rinaldo}. Let $r, r'> 0$. We have
\begin{equation}\label{eq:ASP_random_bound}
	\E \tilde{\delta}(\hat{C},C) \le \frac{8m}{\theta (m - 1)} \times \left[\min \left\{ c^{-1} (2 + \epsilon) \frac{K}{\lambda^2 \alpha_n}\frac{\pi_{\max}}{\pi_{\min}^2 m} C_{n,m,r'}, 1 \right\} + \frac{2K}{n^{r'}} + \frac{1}{m^r} \right] + \pi_{\max} \times e^{ - (1 - \theta)^2 Tp / 2}.
\end{equation}
Here the quantity $C_{n, m, r'}\rightarrow 1$, if $\frac{\pi_{\min}^2 m}{\pi_{\max} \log n} \rightarrow \infty$.
\end{cor}
The proof of Corollary~\ref{cor:asp} follows from Corollary~\ref{cor:m-subgraph} and an estimate for $\E \delta(\hat{Z}_{S}, Z_{S})$ given in (\ref{eq:ASP_random}), which is obtained using Lemma~\ref{lem:lei-rinaldo}. In order for the first term of \eqref{eq:ASP_random_bound} to go to zero we need $\frac{K \pi_{\max}}{\lambda^2 m \alpha_n \pi_{\min}^2} = o(1)$, i.e. $m \gg \frac{K \pi_{\max}}{\lambda^2 \alpha_n \pi_{\min}^2}$. Thus for balanced block sizes (i.e. $\pi_{\max}, \pi_{\min} \asymp \frac{1}{K}$) we need to have $m \gg \frac{K^2}{\lambda^2 \alpha_n}$. So, qualitatively, for large $K$ or small $\alpha_n$ or a small separation between the blocks, $m$ has to be large, which is natural to expect. In particular, for fixed $K$ and $\lambda$, this shows that we need subgraphs of size $m \gg n d_n^{-1}$, and $T \gg \frac{n^2}{m^2}$ many of them to achieve consistency (here the average degree $d_n \asymp n\alpha_n$). Let $T = \frac{n^2}{m^2} r_n$ and $m = \frac{n}{d_n} s_n$, where both $r_n,s_n \rightarrow \infty$. Let us see what computational gain we get from this. Spectral clustering on the full graph has complexity $O(n^3)$, while the complexity of \pace{} with spectral clustering is
\[
	O(Tm^3) = O\left(\frac{n^2}{m^2}r_n m^3\right) = O\left(n^2 m r_n\right) = O\left(\frac{n^3}{d_n}r_n s_n\right).
\]
So if $d_n = \Theta(n^\alpha)$, then the complexity would be $O(n^{3 - \alpha} r_n s_n)$, which is essentially $O(n^{3-\alpha})$. When $d_n = \Theta(\log n)$ the gain is small.

Note however that for a parallel implementation, with each source processing $M$ out of the $T$ subgraphs, we may get a significant boost in running time, at least in terms of constants; the running time would be $O\left(\frac{n^3}{Md_n}r_n s_n\right)$.

\begin{cor}[$(1+\epsilon)$-approximate adjacency spectral clustering with ego subgraphs]\label{cor:asp_ego}  Assume the setting of Lemma~\ref{lem:lei-rinaldo}. Let $r, r' > 0$. We have
\begin{equation}\label{eq:ASP_ego_bound}
	\E \tilde{\delta}(\hat{C},C) \le 
 \frac{32(1 + \theta)^2}{\lambda^2} \times \left[\min\left\{ c^{-1}(2+\epsilon) \frac{K}{\lambda^4 \alpha_n^2}\frac{\pi_{\max}}{\pi_{\min}^2 n} D_{n,B,r'}, 1 \right\} + \frac{2K}{n^{r'}} + \frac{16}{n \lambda^2 \alpha_n} + \left(\frac{4}{n \lambda \alpha_n}\right)^r\right] + \Delta,
\end{equation}
where
\[
\Delta \le  \frac{4}{\lambda^2\alpha_n^2}  \times \exp\left( - \frac{n\lambda^2\alpha_n}{16}\right) + \frac{4}{n^2 \lambda^2 \alpha_n^2} \times \exp \left( - \frac{\theta^2 n \alpha_n}{6} \right) + \pi_{\max} \times \left[2\exp\left(-\frac{n\lambda^4 \alpha_n^2}{16}\right) + \exp\left(- \frac{T \lambda^2\alpha_n^2}{16}\right)\right].
\]
Here the quantity $D_{n,B,r'} \rightarrow 1$, if $\frac{\pi_{\min}^2 \lambda^2 n\alpha_n}{\pi_{\max} \log n} \rightarrow \infty.$
\end{cor}
The proof of Corollary~\ref{cor:asp_ego} follows from 
 Corollary~\ref{cor:ego-subgraph} and an estimate for $\E \delta(\hat{Z}_{S}, Z_{S}) \mathbf{1}_{(|S| \ge m_{\star})}$ given in (\ref{eq:ASP_ego}), which is obtained using Lemma~\ref{lem:lei-rinaldo}. For the right hand side in \eqref{eq:ASP_ego_bound} to go to zero (assuming $K$ fixed, balanced block sizes), we need $\min\{n\alpha_n^2, T\alpha_n^2\} \rightarrow \infty$. In terms of average degree this means that we need $d_n \gg \sqrt{n}$, and $T \gg \frac{n^2}{d_n^2}$. That with ego neighborhoods we can not go down to $d_n = \Theta(\log n)$ is not surprising, since these ego networks are rather sparse in this case. One needs to use larger neighborhoods. Anyway, writing $d_n = \sqrt{n} r_n$, $T = \frac{n^2}{d_n^2} s_n$, where both $r_n,s_n \rightarrow \infty$, the complexity of adjacency spectral clustering, in this case becomes $O(Td_n^3) = O(n^2 d_n  r_n s_n)$ and with $M$ processing units gets further down to $O(\frac{n^2 d_n}{M} r_n s_n)$.

Although from our analysis, it is not clear why \pace{} with spectral clustering should work well for sparse settings, in numerical simulations, we have found that in various regimes \pace{} with (regularized) spectral clustering vastly outperforms ordinary (regularized) spectral clustering (see Table~\ref{table:comparison-rsc-rand}).

It seems that the reason why \pace{} works well in sparse settings lies in the weights $N_{ij}$. With $h$-hop neighborhoods as the chosen subgraphs, if $P_{uv} =\mathbf{1}_{\{\rho_g(u,v) \le h\}}$, where $\rho_g$ is the geodesic distance on $G$, then $N_{ij} = (P^2)_{ij}$. It is known that spectral clustering on the matrix of geodesic distances works well in sparse settings (\cite{bhattacharyya2014community}). \pace{} seems to inherit that property through $N$, although we do not presently have a rigorous proof of this.

We conclude this section with an illustration of \pace{} with random $m$-subgraphs using SDP as the algorithm $\mathcal{A}$. We shall use the setting of Theorem~1.3 of \cite{guedon2014community} for the illustration, stated here with slightly different notation. Let SDP-GV denote the following SDP \citep[SDP (1.10)]{guedon2014community}
\begin{align*}
&\text{maximize } \langle A, X \rangle\\
&\text{subject to } X \succeq 0, X \ge 0, \mathrm{diag}(X) = I_n, \mathbf{1}^{\top} X \mathbf{1} = \mathbf{1}^\top C \mathbf{1}.
\end{align*}

\begin{lemma}[\cite{guedon2014community}, Theorem 1.3]\label{lem:sdp-gv}
Consider an SBM with $\min_{k} B_{kk} \ge a/n$, $\max_{k\ne k'}B_{kk'} \le b/n$, where $a > b$. Also let the expected variance of all edges $\frac{2}{n(n-1)}\sum_{1 \le k\le k' \le K}B_{kk'}(1 - B_{kk'})n_{kk'} = \frac{g}{n}$, where $n_{kk'}$ denotes the number of pairs of vertices, one from community $k$, the other from community $k'$. Fix an accuracy $\epsilon > 0$. If $g \ge 9$ and $(a - b)^2 \ge 484 \epsilon^{-2}g$, then any solution $\hat{X}$ of SDP-GV satisfies
\[
\frac{1}{n^2}\|\hat{X} - C\|_F^2 \le \epsilon.
\]
\end{lemma}	

\begin{cor}[SDP with random $m$-subgraphs]\label{cor:sdp}
Consider the setting of Lemma~\ref{lem:sdp-gv}. Let $c > 1 > c' > 0$. and set $\bar{g}_1 = \frac{c_1(n-1)g}{m - 1}$ and $\bar{g}_2 = \frac{c_2mg}{n}$, where $c_1 = c^2 \left(1 - \frac{n - cm}{cm(n\pi_{\min} - 1)}\right)$ and $c_2 = (c')^2 \left(1 - \frac{n - c'm}{c'm(n\pi_{\max} - 1)}\right)$. Fix an accuracy $\epsilon > 0$. Assume that $(a - b)^2 \ge 484 \epsilon^{-2} \bar{g}_1$, and that $\bar{g}_2 \ge 9$. Then we have
\[
\E \tilde{\delta}(\hat{C},C) \le \frac{4m}{\theta (m - 1)} \times (\epsilon + e^3 5^{-m} + K e^{- (1 - c')^2 m \pi_{\min}/4} + K e^{- (c - 1)^2 m \pi_{\min}/4}) +\frac{1}{2} \times e^{-(1 - \theta)^2 Tp/2}.
\]
\end{cor}
For the simple two parameter blockmodel $B = \frac{1}{n}((a - b)I + b\mathbf{1}\mathbf{1}^\top)$ with equal community sizes, we have $g \asymp \frac{a + (K-1)b}{K} \asymp d_n$, the average degree of the nodes (note that $d_n = \frac{a + (K-1)b}{K} - \frac{a}{n}$). The assumptions of Corollary~\ref{cor:sdp} are satisfied when
\[
m = \Omega\left(\max\left\{\frac{n}{d_n}, \frac{nd_n}{\epsilon^2(a - b)^2}\right\}\right).
\]
This is exactly similar to what we saw for spectral clustering (take $a = n\alpha_n$, and $b = n\alpha_n (1 - \lambda)$). In particular, when the average degree $d_n = \Theta (n^{\alpha})$, and $a - b = \Theta(n^{\alpha})$, we need $m = \Omega(n^{1 - \alpha}/\epsilon^2)$ and $T \gg n^{2\alpha}/\epsilon^4$ for \pace{} to succeed. However, in the bounded degree regime, the advantage is negligible, only from a potentially smaller constant, because we need $m = \Omega(n)$. Again, from our numerical results, we expect that with $h$-hop subgraphs, \pace{} will perform much better.

\subsection{Results on \gale{}}
\label{sec:gale-results}
\hfff{gale_results}
We denote the unnormalized miscustering error between estimated labels $\hat{Z}$ and the true labels $Z$, ($\hat{Z}, Z \in \{0, 1\}^{n \times K}$) of the same set of nodes as $\M(Z, \hat{Z}):= n \delta(Z, Z')= \min_{Q \text{ perm.}} \|\hat{Z}-ZQ\|_0.$ Note that since $\hat{Z},Z$ are binary, the $\|\hat{Z}-ZQ\|_0=\|\hat{Z}-ZQ\|_1=\|\hat{Z}-ZQ\|_F^2$. As discussed earlier, the number of misclustered nodes will be half of this number. %We will use $\pi_k = n_k/n$ as the proportion of nodes from the $k$-th class, $\pi_{\min}:=\min_k \pi_k$.

The main idea of our algorithm is simple. Every approximately accurate clustering of a set of nodes is only accurate up to a permutation, which can never be recovered without the true labels. However we can align a labeling to a permuted version of the truth, where the permutation is estimated from another labeling of the same set of vertices. This is done by calculating the confusion matrix between two clusterings. We call two clusterings aligned if cluster $i$ from one clustering has a large overlap with cluster $i$ from the other clustering. If the labels are ``aligned'' and the clusterings agree, this confusion matrix will be a matrix with large diagonal entries.
This idea is used in the \match{} algorithm, where we estimate the permutation matrix to align one clustering to another. 

Now we present our main result. We prove consistency of a slightly modified and weaker version of Algorithm~\ref{alg:seq}. In Algorithm~\ref{alg:seq}, at every step of a traversal, we apply the \match{} algorithm on the intersection of the current subgraph and the union of all subgraphs previously aligned to estimate the permutation of the yet unaligned current subgraph. However, in the theorem presented below we use the intersection between the unaligned current subgraph with the last aligned subgraph. Empirically it is better to use the scheme presented in Algorithm~\ref{alg:seq} since it increases the size of the intersection which requires weaker conditions on the clustering accuracy of any individual subgraph. We leave this for future work. 

We now formally define our estimator $\hat{Z}^{\gale{}}$. Let $y^{(\ell)}_i = \mathbf{1}_{\{i\in S_\ell\}}$. Let $\hat{Z}^{(\ell)}$ denote the aligned clustering of subgraph $S_\ell$ and let $N_i=\sum_{\ell=1}^T y^{(\ell)}_i$. Define
\begin{equation}\label{est-gale}
	\hat{Z}^{\gale{}}_{ik}:=\frac{\sum_{\ell=1}^T y^{(\ell)}_i \hat{Z}^{(\ell)}_{ik}}{N_i} \mathbf{1}_{\{N_i \ge \tau\}}.
\end{equation}
The entries of $\hat{Z}^{\gale{}}$ will be fractions, but as we show in Lemma~\ref{lem:rounding}, rounding it to a binary matrix will not change consistency properties.

Note that \gale{} depends on the spanning tree we use and particular the traversal of that spanning tree. Let ${\sf SpanningTrees}_{G}$ be the set of all spanning trees of a graph $G$. For $\mathcal{T}\in {\sf SpanningTrees}_{G}$, let ${\sf Traversals}_{\mathcal{T}}$ be the set of all traversals of $\mathcal{T}$. Let $\hat{Z}^{\gale}_{\mathcal{T},(x_1, \ldots, x_J)}$ be the outcome of \gale{} on the traversal $(x_1, \ldots, x_J)$ of $\mathcal{T} \in {\sf SpanningTrees}_{\mathcal{S}_{m, T}}$.

\begin{theorem}[Misclustering rate of \gale{}]\label{thm:gale-main}
	Let $0 < \theta < 1$ and $r, r' > 0$. Let $m = \Omega_{r,r',\theta}\left(\sqrt{\frac{n \log n}{\pi_{\min}}}\right)$, $T=\Omega_{r,r',\theta}(n\log n/m)$, and $\tau = \frac{\theta T m}{n}$. Consider an algorithm $\mathcal{A}$ which labels any random $m$-subgraph with error $\le m_1\pi_{\min}/\cp$ with probability at least $1-\delta$. Then we have, with probability at least $1 - T\delta - O(1/n^{r'})$, that 
	\begin{equation}\label{bound-gale}
		\max_{\mathcal{T} \in {\sf SpanningTrees}_{\mathcal{S}_{m, T}}} \max_{(x_1, \ldots, x_J) \in {\sf Traversals}_{\mathcal{T}}} \delta(\hat{Z}^{\gale}_{\mathcal{T},(x_1, \ldots, x_J)}, Z) \le \frac{1}{\theta T} \sum_{\ell = 1}^T \delta(\hat{Z}_\ell,Z) + O\left(\frac{1}{n^r}\right).
	\end{equation}
\end{theorem}
Again, the constant $\theta$ can be taken as close to $1$ as one desires. Thus the above bound is also essentially optimal. 

We will now illustrate Theorem~\ref{thm:gale-main} with several algorithms $\mathcal{A}$. We begin with a result on $(1+\epsilon)$-approximate adjacency spectral clustering. 
\begin{cor}[$(1+\epsilon)$-approximate adjacency spectral clustering with \gale{}]\label{cor:m-subgraph-gale}
 Assume the setting of Lemma~\ref{lem:lei-rinaldo}. Let $0 < \theta < 1$. Let $r, r', r'', r''' > 0$. Let $m = \Omega_{r,r',\theta}\left(\sqrt{\frac{n \log n}{\pi_{\min}}}\right)$, $T=\Omega_{r,r',\theta}(n\log n/m)$, and $\tau = \frac{\theta T m}{n}$. Then we have, with probability at least $1 - \frac{T}{m^{r''}} - \frac{2TK}{n^{r'''}} - O\left(\frac{1}{n^{r'}}\right),$ that
\begin{equation}\label{bound-gale-spectral}
		\max_{\mathcal{T} \in {\sf SpanningTrees}_{\mathcal{S}_{m, T}}} \max_{(x_1, \ldots, x_J) \in {\sf Traversals}_{\mathcal{T}}} \delta(\hat{Z}^{\gale}_{\mathcal{T},(x_1, \ldots, x_J)}, Z) \le \frac{2}{\theta}\left[\min \left\{ c^{-1} (2 + \epsilon) \frac{K}{\lambda^2 \alpha_n}\frac{\pi_{\max}}{\pi_{\min}^2 m} C_{n,m,r'''}, 1 \right\} \right]  + O\left(\frac{1}{n^{r}}\right),
	\end{equation}
	where the constant $C_{n,m, r'''}$ is the same as in Corollary~\ref{cor:m-subgraph}
\end{cor}
 We see that the first term is exactly same as the first term in Corollary~\ref{cor:m-subgraph}. This, for balanced graphs, again imposes the condition $m \gg \frac{K^2}{\lambda^2 \alpha_n}$. In particular, if $K = \Theta(1)$ and we are in a dense well separated regime, with $\lambda = \Theta(1)$, $\alpha_n=\Omega(1/\sqrt{n})$, then we need $m=\Omega(\sqrt{n\log n})$. If $K = \Theta(1)$, $\lambda = \Theta(1)$ and $\alpha_n = \Theta(\log n/n)$, then we need $m \gg n/\log n$. In both cases, we need $T = \Omega(n\log n/m)$. Thus in the regime where average degree is like $\log n$ there is still some computational advantage for very large networks (also factoring in parallelizability); however, for moderately sized networks, \gale{} may not lead to much computational advantage. 

Now we present an exact recovery result with SDP as the base algorithm $\mathcal{A}$. We shall use a result\footnote[2]{We are not using Lemma~\ref{lem:sdp-gv} as it only shows that the solution of SDP-GV has small norm difference from the ideal clustering matrix, but does not relate this directly to misclustering error.} from~\cite{yan2017exact} on an SDP which they call SDP-$\lambda$. Let $\kappa:=\pi_{\max}/\pi_{\min}$. Also let $\boldsymbol{\pi}$ denote the vector of the cluster proportions $(\pi_1,\dots,\pi_K)$. \hfff{cluster_vec}
\begin{lemma}[Theorem 2 of \cite{yan2017exact}]	\label{lem:exact-sdp}
	Let $r > 0$. Then $Z\,\mathrm{diag}(n\boldsymbol{\pi})^{-1}Z^\top$ is the optimal solution of the SDP-$\lambda$, with probability at least $1-O((n\pi_{\min})^{-r})$, if 
	\begin{equation}
	\label{eq:sdp-separation}
		\min_k (B_{kk}-\max_{\ell\ne k}B_{k\ell}) = \tilde{\Omega}_r\left( \kappa \max_{k} \sqrt{\frac{\max(B_{kk},K\max_{\ell\neq k}B_{k\ell}) }{n\pi_k}} \right).
	\end{equation}
\end{lemma}
Assuming that any subsequent clustering of the exactly recovered scaled clustering matrix $Z\,\mathrm{diag}(n\boldsymbol{\pi})^{-1}Z^\top$ gives the exact clustering $Z$ back (for example, our distance based naive algorithm \naivecluster{}\footnote[3]{detailed in Appendix~\ref{sec:app_greedy}} will do this), we have the following corollary.
\begin{cor}[SDP with \gale{}, exact recovery]\label{cor:m-subgraph-gale-sdp}
	Assume the setting of Lemma~\ref{lem:exact-sdp}. Let $0 < \theta < 1$. Let $r, r', r'', r''' > 0$. Let $m = \Omega_{r,r',\theta}\left(\sqrt{\frac{n \log n}{\pi_{\min}}}\right)$, $T=\Omega_{r,r',\theta}(n\log n/m)$, and $\tau = \frac{\theta T m}{n}$. Then, as long as the separation condition in \eqref{eq:sdp-separation} holds with $m$ replacing $n$, we have, with probability at least $1 - O\left(\frac{T}{(m\pi_{\min})^{r''}} + \frac{TK}{n^{r'''}} + \frac{1}{n^{r'}}\right)$, that
	\begin{equation}\label{bound-gale-sdp}
	\max_{\mathcal{T} \in {\sf SpanningTrees}_{\mathcal{S}_{m, T}}} \max_{(x_1, \ldots, x_J) \in {\sf Traversals}_{\mathcal{T}}} \delta(\hat{Z}^{\gale}_{\mathcal{T},(x_1, \ldots, x_J)}, Z) =  O\left(\frac{1}{n^{r}}\right),
	\end{equation}
\end{cor}
Note that, in the above bound $r$ can taken to be  greater than $1$. This means that, with high probability, the proportion of misclustered nodes is less than $1/n$ and hence zero, leading to exact recovery. As for computational complexity, note that the separation condition \eqref{eq:sdp-separation}, with $n$ replaced by $m$, restricts how small $m$ can be. Consider the simple SBM \eqref{eq:simple_sbm} with balanced block sizes for concreteness. In this case, the separation condition essentially dictates, as in the case of spectral clustering, that $m \gg \frac{K^2}{\lambda^2\alpha_n}$. Thus the remarks made earlier on how large $m$ or $T$ should be chosen apply here as well.

 As discussed earlier in Section~\ref{sec:local_algo}, even a naive implementation of \gale{} will only result in an $O(n^{3/2})$ running time in addition to the time ($\eta_{m,T}$) required to cluster the $T$ random $m$-subgraphs, whereas a more careful implementation will only add a time to $\eta_{m,T}$ that is nearly linear in $T$. Since SDPs are notoriously time intensive to solve, this gives us a big saving. 

%One can also do similar analyses for sparse graphs. We do not include these here, since most published results on sparse networks which show that the misclustering error is a small constant governed by the separation (see, e.g., \cite{le2015sparse}) have typically been proved in the balanced setting and cannot be applied directly to the case of random subgraphs.

%%%%%%%%%%%%%%%%%%%%%%%%%%%%%%%%%%%%%%%%%%%%%%%%%%%%%%%%
\section{Simulations and real-data analysis}\label{sec:simulations}
In Table~\ref{table:comparison} we present a qualitative comparison of \pace{} and \gale{} with four representative global community detection methods Profile Likelihood (PL), Mean Field Likelihood (MFL), Spectral Clustering (SC) and Semi Definite Programming (SDP).

\begin{table}[htbp!]
\centering
\begin{tabular}{|c|c|c|c|c|c|c|}\hline
   		& PL 	 & MFL 	  & SC     & SDP 	& \pace{} & \gale{} \\ \hline \hline

% regime & & & & & & \\  \hline 

% sparse  ($d_n = O(1)$)	& \xmark & \xmark & \xmark & \cmark & \cmark & \cmark \\
% semi-sparse ($d_n \rightarrow \infty, d_n = O(\log n)$) & \xmark & \xmark & \xmark & \cmark & \cmark & \cmark \\
% semi-dense ($d_n \gg \log n$) & \cmark & \cmark & \cmark & \cmark & \cmark & \cmark \\
% dense ($d_n = \Omega(n)$)		& \cmark & \cmark & \cmark & \cmark & \cmark & \cmark \\
% \hline \hline

%computation & & & & & & \\ \hline 
Computationally easy & \xmark & \xmark & \cmark & \xmark & \cmark & \cmark \\ \hline
Theoretical complexity & NP hard & $O(n^{\theta_1})$\tablefootnote[2]{\label{note1} $\theta_j$'s ($\ge 3$) depend on details of implementation and numerical accuracy sought.} & $O(n^3)$ & $O(n^{\theta_2})$\textsuperscript{\ref{note1}} & $O(n^{2+\epsilon})$\tablefootnote[3]{\label{note2} with $O(n^3)$ algorithms} & $O(n^{2+\epsilon})$\textsuperscript{\ref{note2}} \\ \hline
Real-world scalability $(n)$ & $10^2 - 10^3$ & $10^2 - 10^3$ & $10^6$ & $10^2 - 10^3$ & $\gg 10^6$ & $\gg 10^6$ \\ \hline
Parallelizability & \xmark & \xmark & \xmark & \xmark & \cmark & \cmark \\ \hline
\end{tabular}
\caption{Qualitative comparison of various methods.}
\label{table:comparison}
\end{table}

\subsection{Simulations: comparison against traditional algorithms}
For simulations we will use of the following simple block model:
\[
B =(p - q) I + q J = \rho_n a((1 - r) I + r J),
\]
where $I$ is the $K$ dimensional identity matrix and $J$ is the $K\times K$ matrix of all ones. Here $\rho_n$ will be the degree density, and $r$ will measure the relative separation between the within block and between block connection probabilities, i.e. $p$ and $q$. If the blocks have prior probabilities $\pi_i, i = 1, \ldots, K$, then the average degree $d$, under this model is given by
\[
d_n = (n - 1) \rho_n a ((1 - r)\sum_i \pi_i^2 + r).
\]
In particular, if the model is balanced, i.e. $\pi_i = 1/K$ for all $i$, then
\[
d_n = \frac{(n - 1) \rho_n a(1 + (K - 1)r)}{K}.
\]
%\textbf{Setting 1:} unbalanced 2-block model, moderate separation between $p$ and $q$. In this setting we fix $\pi_1 = 0.7$ and $\pi_2 = 0.3$, $n = 1000$, $a = 1$, $r = 0.1$, and vary $\rho$.

In order to understand and emphasize the role of \pace{} and \gale{} in reducing computational time while maintaining good clustering accuracy, we use different settings of sparsity for different methods. For recovering $\hat{Z}$ from $\hat{C}$ in \pace{}, we have used random projection plus $K$-means (abbreviated as \rpk{} below), and spectral clustering (SC). We also want to point out that, for sparse unbalanced networks \gale{} may return more than $K$ clusters, typically when a very small fraction of nodes has not been visited. However, it is possible that the unvisited nodes have important information about the network structure. For example, all subgraphs may be chosen from the larger clusters, thereby leaving the smallest cluster unvisited. We take care of this by computing the smallest error between the $(K+1)!$ permutations of \gale{}'s clustering to the ground truth. This essentially treats the smallest cluster returned by \gale{} as misclustered. In real and simulated networks we have almost never seen \gale{} return a large number of unvisited nodes.

\vskip10pt
\noindent
\textbf{SDP with ADMM:}
Interior point methods for SDPs are not very fast in practice. We have solved SDPs using an ADMM based implementation of \cite{yan2016covariate}. From Table~\ref{table:comparison-sdp} we see that \pace{} and \gale{} significantly reduces the running time of SDP without losing accuracy too much. In fact, if we use spectral clustering to estimate $\hat{Z}$ from $\hat{C}$ in the last step of \pace{}, we get zero misclustering error (ME). 

\vskip10pt
\noindent
\textbf{Mean Field Likelihood:}
From Table~\ref{table:comparison-mf} we see that our implementation of mean field on the full graph did not converge to an acceptable solution even after five and half hours, while both \pace{} and \gale{} return much better solutions in about two minutes. In fact, with spectral clustering in the last step of \pace{}, the misclustering error is only 0.14, which is quite good. This begs the question if this improvement is due to spectral clustering only. We show in the next simulation that in certain settings, even when spectral clustering is used as the base algorithm, \pace{} and \gale{} lead to significant improvements in terms of accuracy and running time. 

\begin{table}[!htb]
	\hspace{-1em}
\begin{minipage}{.5\textwidth}
%\begin{tabular}[htbp!]
	\begin{tabular}{|c||c|c|}\hline
		Algorithm & ME(\%) & Time taken \\ \hline \hline
		SDP  & 0 & 1588s \\
		SDP + \pace{} + SC & 0 & 288s  \\
		SDP + \pace{} + \rpk{} & 9.1 & 281s  \\
		SDP + \gale{} & 1.2 &  281s\\ \hline
	\end{tabular}
\caption{\pace{} and \gale{} with SDP as the base method. Simulation settings: $n = 5000$, average degree $= 128$, $m=500$, 4 equal sized clusters, $T=100$, parallel implementation in Matlab with 20 workers.}
\label{table:comparison-sdp}
\end{minipage}	
\hspace{2em}
\begin{minipage}{.5\textwidth}
	\vspace{1em}
	%\begin{tabular}[htbp!]
	\begin{tabular}{|c||c|c|}\hline
		Algorithm & ME(\%) & Time taken \\ \hline \hline
		MFL  & 50 & 20439s \\
		MFL + \pace{} + SC &  1.4 & 131s \\
		MFL + \pace{} + \rpk{} & 36.5 & 125s  \\
		MFL + \gale{} & 19.2 &  126s\\ \hline
	\end{tabular}
	\caption{\pace{} and \gale{} with Mean Field Likelihood (MFL) as the base method. Simulation settings: $n = 5000$, average degree $= 13$, 2-hop neighborhood, 2 equal sized clusters, $T=100$, parallel implementation in Matlab with 20 workers.}
	\label{table:comparison-mf}
\end{minipage}	
\end{table}
\newcommand\T{\rule{0pt}{2.6ex}}       % Top strut
\newcommand\B{\rule[-1.2ex]{0pt}{0pt}} % Bottom strut
\vskip10pt
\noindent
\textbf{Regularized Spectral Clustering}: In sparse unbalanced settings, regularized spectral clustering with \pace{} and \gale{} performs significantly better than regularized spectral clustering on full graph. In fact, with spectral clustering used in the last step of \pace{}, we can hit about 5\% error or below, which is quite remarkable. See Table~\ref{table:comparison-rsc-rand}. In Section~\ref{sec:realdata} we will see that \pace{} and \gale{} also add stability to spectral clustering (in terms of clustering degree 1 vertices).

\begin{table}[!h]
	\centering
%	\begin{minipage}{.3\textwidth}
		%\begin{tabular}[htbp!]
		\begin{tabular}{|c||c|c||c|c||c|c|}\hline
			&\multicolumn{2}{c||}{Random 1500-subgraph}&\multicolumn{2}{c||}{3-hop neighborhood}&\multicolumn{2}{c|}{5-hop onion}\T\B\\
			\hline
			Algorithm & ME(\%) & Time taken	& ME(\%) & Time taken	& ME(\%) & Time taken \T\B\\ \hline \hline
			RSC  & 39.6 & 87s & & & & \T\B\\
			RSC + \pace{} + SC & 11.1 & 26s & 3.4 & 21s & 5.1 & 59s \T\B\\
			RSC + \pace{} + \rpk{} &  34.7 & 20s & 34.2 & 14s & 18 & 53s \T\B\\
			RSC + \gale{} & 17.9 &  23s & 33.6 & 13s & 29.7 & 52s \T\B\\ \hline
		\end{tabular}
		\caption{\pace{} and \gale{} with Regularized Spectral Clustering (RSC) as the base method. Simulation settings: $n = 5000$, average degree $= 7$, 2 unequal sized clusters with relative sizes $\boldsymbol{\pi}=(0.2, 0.8)$, $T = 100$, parallel implementation in Matlab with 20 workers.}
		\label{table:comparison-rsc-rand}
%	\end{minipage}
\end{table}

\vskip10pt
\noindent
\textbf{Profile Likelihood with tabu search:} Optimizing profile likelihood (PL) or likelihood modularity (\cite{bickel2009nonparametric}) for community detection is a combinatorial problem, and as such hard to scale, even if we ignore the problem of local minima. In Table~\ref{table:comparison-pl} we compare running time of profile likelihood (optimized using tabu search) and its divide and conquer versions. We see that the local methods significantly cut down the running time of PL without losing accuracy too much.

We also applied profile likelihood on 5000 node graphs with 20 workers. Although \pace{} and \gale{} finished in about 22 minutes,
% with accuracies ranging between 25-35\%, 
the global method did not finish in 3 days. So, here we present results on 1000 node networks instead.

\begin{table}[!htb]
	%	\begin{minipage}{.3\textwidth}
	%\begin{tabular}[htbp!]
	\centering
	\begin{tabular}{|c||c|c||c|c|}\hline
		&\multicolumn{2}{c||}{Random 310-subgraph} &\multicolumn{2}{c|}{2-hop neighborhood}\T\B\\
		\hline
		Algorithm & ME(\%) & Time taken	& ME(\%) & Time taken \T\B\\ \hline \hline
		PL  & 0 & 70m & & \T\B\\
		PL + \pace{} + \rpk{} & 3.9 & 30m & 3.5 & 38m \T\B\\
		PL + \gale{} & 1.2 & 30m & 29.5 & 38m\T\B\\ \hline
	\end{tabular}
	\caption{\pace{} and \gale{} with Profile Likelihood (PL) as the base method. Simulation settings: $n = 1000$, average degree $= 18.47$, 2 unequal sized clusters with relative sizes $\boldsymbol{\pi}=(0.4, 0.6)$, $T = 50$, parallel implementation in Matlab with 12 workers. We sampled $2$-hop neighborhoods by selecting their roots uniformly at random from nodes having degree greater than the $0.1$th lower quantile ($= 12$) of the degree distribution (average neighborhood size was $310$). Ordinary \pace{} with such a scheme may be thought of as \wpace{}, as discussed in Section~\ref{sec:local_algo}.}
	\label{table:comparison-pl}
	%	\end{minipage}
\end{table}

\begin{remark}
We have seen from the results presented in this section that, for recovering $\hat{Z}$ from $\hat{C}$ in \pace{}, spectral clustering outperforms the random projection based algorithms (e.g., \rpk{}). For smaller networks, this is not an issue (e.g., spectral clustering on the dense $5000\times 5000$ matrix $\hat{C}$ in the context of Table~\ref{table:comparison-mf} took only about 7-8 seconds). However, for networks of much larger scale (say, with several million nodes), that last step would be costly if spectral clustering is used. Designing better algorithms for recovering $\hat{Z}$ from $\hat{C}$ is something we are working on currently.
\end{remark}

\subsection{Real data analysis}\label{sec:realdata}
\begin{figure}[!h]
	\centering                    
	\includegraphics[width = .7\textwidth]{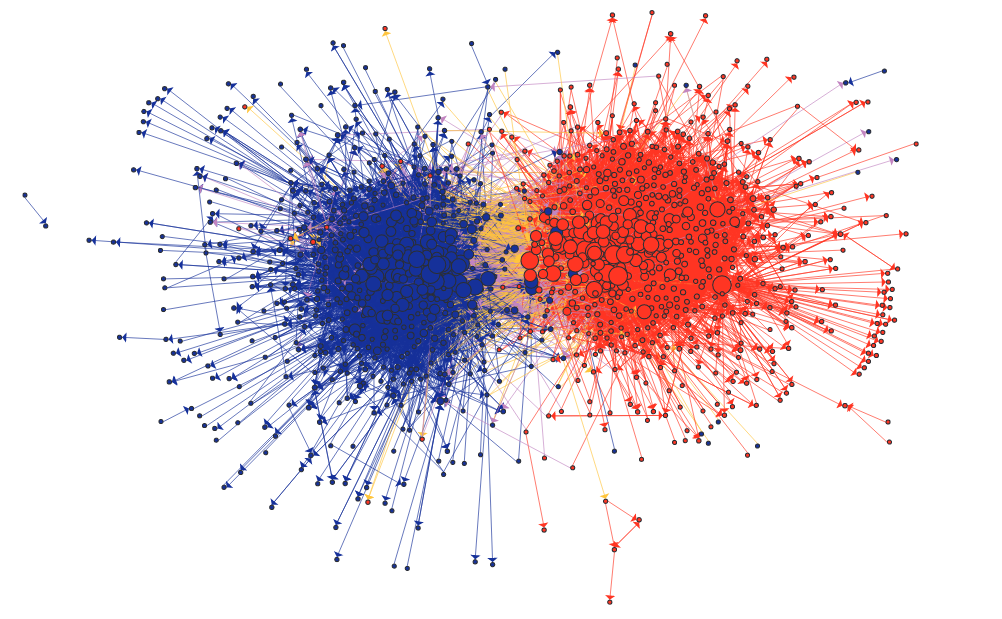}
	\caption{Network of political blogs, reds are conservative, blues are liberal; picture courtesy: \cite{adamic2005political}.}
	\label{fig:polblog}
\end{figure}
\textbf{Political blog data:} This is a directed network (see Figure~\ref{fig:polblog}) of hyperlinks between $1490$ blogs (2004) that are either liberal or conservative (\cite{adamic2005political}); we have ground truth labels available for comparison, $758$ are liberal, $732$ are conservative. We convert it into an undirected network by putting an edge between blogs $i$ and $j$ if there is at least one directed edge between them.
	
 The resulting network has lots of isolated nodes and isolated edges. The degree distribution is also quite heterogeneous (so a degree-corrected model would be more appropriate). We focus on the largest connected component. We use Laplacian spectral clustering (row normalized, to correct for degree heterogeneity), with \pace{}. 
\begin{table}[!htbp]
	\centering
	\begin{tabular}{|c||c|c|c|c|}\hline
		Largest Conn. Comp. & RSC & RSC + \pace{} & SC & SC + \pace{}\\ \hline \hline
		With leaves (1222 nodes) & 18.74\% & 6.79\% & 48.12\% & 6.55\%\\ \hline
		Without leaves (1087 nodes) & 11.87\% &  4.23\% & 3.13\% &  3.86\%\\ \hline
	\end{tabular}
	\caption{Misclustering rate in the political blog data. \pace{} was used with $T = 10$, and $h$-hop neighborhoods with $h = 2$, with roots chosen at random from high degree nodes.}
	\label{table2}
\end{table}
\begin{table}[!htbp]
	\centering
	\begin{tabular}{|c||c|c|c|c|c|c|}\hline
		Largest Conn. Comp. & RSC & RSC + \pace{}&RSC + \gale{} & SC & SC + \pace{}&SC + \gale{}\\ \hline \hline
		With leaves (1222 nodes) & 18.74\% & 13.34\% &  11.62\% &48.12\% & 7.86 \% & 5.81\%\\ \hline
		Without leaves (1087 nodes) & 11.87\% &  12.8\% & 10.0\% &  4.23\% & 7.28\% &  6.7\%\\ \hline
	\end{tabular}
	\caption{Misclustering rate in the political blog data. \gale{} and \pace{} were used with $T = 50$, and $m=300$ random subgraphs.}
	\label{table3}
\end{table}
Tables~\ref{table2}-\ref{table3} show that \pace{} and \gale{} add stability (possibly in eigenvector computation) to spectral clustering. Indeed, with \pace{} and \gale{} we are able to cluster ``leaf'' vertices (i.e. vertices of degree $1$), with significantly more accuracy.  

%%%%%%%%%%%%%%%%%%%%%%%%%%%%%%%%%%%%%%%%%%%%%%%%%%%%%%%%
\section{Discussion}\label{sec:discussions}
To summarize, we have proposed two divide-and-conquer type algorithms for community detection, \pace{} and \gale{}, which can lead to significant computational advantages without sacrificing accuracy. The main idea behind these methods is to compute the clustering for each individual subgraph and then ``stitch'' them together to produce a global clustering of the entire network. The main challenge of such a stitching procedure comes from the fundamental problem of unidentifiability of label assignments. That is, if two subgraphs overlap, the clustering assignment of a pair of nodes in the overlap may be inconsistent between the two subgraphs.
%\rd In other words, the label assigned to a ground truth cluster does not matter, as long as all the nodes in the same ground truth cluster get the same label. \bk

 \pace{} addresses this problem by estimating the clustering matrix for each subgraph and then estimating the global clustering matrix by averaging over the subgraphs. \gale{} takes a different approach by using overlaps between two subgraphs to calculate the best alignment between the cluster memberships of nodes in the subgraphs. We prove that, in addition to being computationally much more efficient than base methods which typcally run in $\Omega(n^2)$ time, these methods have accuracy at least as good as the base algorithm's typical accuracy on the type of subgraphs used, with high probability. Experimentally, we show something more interesting --- we identify parameter regimes where a local implementation of a base algorithm based on \pace{} or \gale{} in fact outperforms the corresponding global algorithm. One example of this is the Meanfield algorithm, which typically suffers from bad local optima for large networks. Empirically, we have seen that on a smaller subgraph, with a reasonable number of restarts, it finds a local optima that is often highly correlated with the ground truth. \pace{} and \gale{} take advantage of this phenomenon to improve on accuracy/running time significantly. Another example is Regularized Spectral Clustering on sparse unbalanced networks. We intend to theoretically investigate this further in future work.
 
 Finally, working with many subgraphs naturally leads to the question of self consistency of the underlying algorithm. This is often crucial in real world clustering problems with no available ground truth labels. We intend to explore this direction further for estimating model parameters like the number of clusters, algorithmic parameters like the size and number of subgraphs, number of hops to be used for the neighborhood subgraphs, etc. Currently, these are all picked a priori based on the degree considerations. It may also be possible to choose between different models (e.g., standard blockmodels, degree corrected models, dot product models etc.) by examining which model leads to the most self consistent results. We leave this for future work.
 	
 In conclusion, not only are our algorithms, to the best of our knowledge, the first ever divide-and-conquer type algorithms used for community detection, we believe that the basic principles of our methods will have a broad impact on a range of clustering and estimation algorithms that are computationally intensive. 

%%%%%%%%%%%%%%%%%%%%%%%%%%%%%%%%%%%%%%%%%%%%%%%%%%%%%%%%
\section{Proofs}\label{sec:proofs}
\subsection{Results on \pace{}}
\begin{proof}[Proof of Proposition~\ref{prop:comparison_ineq}]
Since both $Z, Z'$ are $0,1$-valued, we can safely replace the count by Frobenius norm squared, i.e.
\[
\delta(Z,Z') = \inf_{Q \text{ perm.}} \frac{1}{n}\|ZQ - Z'\|_F^2.
\]
Now, note that $(ZQ)(ZQ)^\top = ZZ^\top$ for all permutation matrices $Q$. Thus
\begin{align*}
	\|C - C'\|_F &= \|(ZQ)(ZQ)^\top - Z'Z'^\top\|_F \\
	& = \|(ZQ - Z')(ZQ)^\top + Z'((ZQ)^\top -Z'^\top\|_F \\
	&\le \|(ZQ - Z')(ZQ)^\top\|_F + \|Z'((ZQ)^\top -Z'^\top\|_F \\
	&\le \|ZQ - Z'\|_F \|ZQ^\top\|_2 + \|Z'\|_2 \|(ZQ)^\top -Z'^\top\|_F.
\end{align*}
But $\|Z'\|_2^2$ is the maximum eigenvalue of $Z'^\top Z'$ which is diagonal with its maximum diagonal entry being the size of the largest cluster under $Z'$. Thus $\|Z'\|_2^2$ equals the size of the largest cluster under $Z'$ and so is trivially upper bounded by $n$. Same goes for $\|ZQ^\top\|_2^2$. Therefore we get
\[
\|C - C'\|_F \le 2\sqrt{n}\|ZQ - Z'\|_F.
\]
Squaring this, and taking infimum over all permutation matrices $Q$ in the right hand side, we obtain the claimed inequality.
\end{proof}
Now we will prove Theorem~\ref{thm:pace-main}. The proof will be broken down into two propositions. First we decompose
\begin{equation}\label{eq:decomp1}
\hat{C}_{ij} - C_{ij} = \underbrace{\hat{C}_{ij} - C_{ij}\mathbf{1}_{\{N_{ij} \ge \tau\}}}_{=:(E_1)_{ij}} + \underbrace{C_{ij}\mathbf{1}_{\{N_{ij} \ge \tau\}} - C_{ij}}_{=:(E_2)_{ij}}
\end{equation}
Note that $(E_1)_{ij} = \mathbf{1}_{\{N_{ij} \ge \tau\}} \sum_{\ell = 1}^T y_{ij}^{(\ell)}(\hat{C}_{ij}^{(\ell)} - C_{ij})/N_{ij}$, and $(E_2)_{ij} = - \mathbf{1}_{\{N_{ij} < \tau\}} C_{ij}$, so that $(E_1)_{ij} (E_2)_{ij} = 0$. Therefore
\begin{equation}\label{eq:decomp2}
\frac{1}{n^2}\|\hat{C} - C\|_F^2 = \frac{1}{n^2} (\|E_1\|_F^2 + \|E_2\|_F^2).
\end{equation}
We will estimate $\|E_1\|_F$ and $\|E_2\|_F$ separately.
\begin{prop}\label{prop:genbound_E1}
We have
\begin{equation}
	\E\|E_1\|_F^2 \le \frac{T}{\tau} \E\|\hat{C}^{(S)} - C^{(S)}\|_F^2.
\end{equation}
\end{prop}

\begin{proof}
Let $W_{ij} := \frac{\mathbf{1}_{\{N_{ij} \ge \tau\}}}{N_{ij}}$. Then $ (E_1)_{ij} = W_{ij} \sum_{\ell=1}^T y_{ij}^{(\ell)}(\hat{C}_{ij}^{(\ell)} - C_{ij})$. So, by an application of Cauchy-Schwartz, we have
\begin{align*}
\|E_1\|_F^2 = \sum_{i,j} W_{ij}^2 \left(\sum_{\ell=1}^T y_{ij}^{(\ell)}(\hat{C}_{ij}^{(\ell)} - C_{ij}) \right)^2 &\overset{\text{(C-S)}}{\le}  \sum_{i,j} W_{ij}^2 \underbrace{\left(\sum_{\ell=1}^T y_{ij}^{(\ell)}\right)}_{= N_{ij}} \left(\sum_{\ell=1}^T y_{ij}^{(\ell)}(\hat{C}_{ij}^{(\ell)} - C_{ij})^2\right)\\
& = \sum_{\ell = 1}^T \sum_{i,j} W_{ij}  y_{ij}^{(\ell)}(\hat{C}_{ij}^{(\ell)} - C_{ij})^2\\
&\le \max_{ij} W_{ij} \sum_{\ell = 1}^T \sum_{i,j} y_{ij}^{(\ell)}(\hat{C}_{ij}^{(\ell)} - C_{ij})^2\\
& = \max_{ij} W_{ij} \sum_{\ell = 1}^T \|\hat{C}^{(\ell)} - C^{(\ell)}\|_F^2.
\end{align*}
Note that $W_{ij} = \frac{\mathbf{1}_{\{N_{ij} \ge \tau\}}}{N_{ij}} \le \frac{1}{\tau}$. On the other hand, since the subgraphs were chosen independently using the same sampling scheme, the $\|\hat{C}^{(\ell)} - C^{(\ell)}\|_F$ are identically distributed. Therefore, taking expectations we get
\[
\E \|E_1\|_F^2 \le \frac{1}{\tau} \times \sum_{\ell = 1}^T \E\|\hat{C}^{(\ell)} - C^{(\ell)}\|_F^2 = \frac{T}{\tau} \E\|\hat{C}^{(S)} - C^{(S)}\|_F^2,
\]
where $S$ is a randomly chosen subgraph under our subgraph selection scheme.
\end{proof}

\begin{prop}\label{prop:genbound_E2}
Let $n_{\max}$ be the size of the largest block. Then we have
\begin{equation}
\E\|E_2\|^2 \le n_{\max} n \max_{i,j}\P(N_{ij} < \tau).
\end{equation}
\end{prop}

\begin{proof}
Since $-(E_2)_{ij} = C_{ij}\mathbf{1}_{\{N_{ij} < \tau\}}$, we have $\|E_2\|_F^2 = \sum_{i,j} C_{ij}^2 \mathbf{1}_{\{N_{ij} < \tau\}}$, and by taking expectations we get
\begin{align*}
	\E \|E_2\|_F^2 = \sum_{i,j} C_{i,j}^2 \P(N_{ij} < \tau) & \le \max_{i,j} \P(N_{ij} < \tau) \sum_{i,j} C_{i,j}^2\\
	& = \max_{i,j} \P(N_{ij} < \tau) \sum_{a=1}^K n_a^2 \le n_{\max}n \times \max_{i,j} \P(N_{ij} < \tau).
\end{align*}
\end{proof}

\begin{proof}[Proof of Theorem~\ref{thm:pace-main}]
Combining Propositions~\ref{prop:genbound_E1} and \ref{prop:genbound_E2}, we get \eqref{bound-pace}. Finally, note that
\begin{align*}
 	\E\|\hat{C}^{(S)} - C^{(S)}\|_F^2 &= \E\|\hat{C}^{(S)} - C^{(S)}\|_F^2 \mathbf{1}_{(|S| < m_{\star})} + \E\|\hat{C}^{(S)} - C^{(S)}\|_F^2 \mathbf{1}_{(|S| \ge m_{\star})}\\
 	&\le n^2 \P(|S| < m_{\star}) + 4\E |S|^2 \delta(\hat{Z}_{S}, Z_{S}) \mathbf{1}_{(|S| \ge m_{\star})}.
\end{align*}
\end{proof}

\begin{proof}[Proof of Corollary~\ref{cor:m-subgraph}]
For this sampling scheme $|S| = m \ge m_{\star}$ and with $p = \frac{m(m-1)}{n(n-1)}$, $N_{ij} \sim$ Binomial$(T, p)$ so that we have, using the Chernoff bound\footnote{We use a slightly loose but convenient form of the Chernoff bounds: (i) $\P(X \le (1 - \delta)\mu) \le \exp(-\delta^2\mu/2)$, and (ii) $\P(X \ge (1 + \delta)\mu) \le \exp(-\delta^2\mu/3)$, where $X = X_1 + \cdots + X_n$ and $X_i$'s are independent binary random variables, with $\E X = \mu$, and $0 < \delta < 1$.
} for binomial lower tail, that
\[
	P(N_{ij} < \theta Tp) \le e^{- (1 - \theta)^2 Tp/2}.
\]
Finally, we get (\ref{bound1:random}) by plugging in these parameter values and estimates in (\ref{bound-pace}).
\end{proof}

\begin{proof}[Proof of Corollary~\ref{cor:ego-subgraph}]
The most crucial thing to observe here is that if one removes the root node and its adjacent edges from a $1$-hop neighborhood, then the remaining ``ego network'' has again a blockmodel structure. Indeed, let $S$ be a random ego neighborhood of size $\ge s$ with root $R$, i.e. $V(S) = \{j : A_{Rj} = 1\}$. Then conditional on $V(S)$ being $R$'s neighbors, and the latent cluster memberships, edges in $E(S)$ are independently generated, i.e. for $j,k,\ell,m\in V(S)$, and $s,t\in \{0,1\}$, we have
	\begin{align*}
	\P(A_{jk}=s,A_{\ell m}=t|S, Z) = \P(A_{jk}=s|Z) \P(A_{\ell m}=t|Z).
	\end{align*}

This is because the ``spoke'' edges $A_{Rj}$ are independent of $A_{j,k}, j,k\in V(S)$. Therefore, conditional on $S$, this ego-subgraph is one instantiation of a block model with the same parameters on $|S|$ vertices. 

Now for ego networks, $y_{ij}^{(\ell)} \sim {\sf Bernoulli} (n_{ij}/n)$, where $n_{ij}$ is the total number of ego-subgraphs containing both $i$ and $j$. Notice that
\[
n_{ij} = \sum_{\ell \ne i, j} \mathbf{1}_{\{A_{i\ell} = 1, A_{j \ell} = 1\}},
\]
that is, $n_{ij}$ is the sum of $(n-2)$ independent Bernoulli random variables 
\[
\mathbf{1}_{\{A_{i\ell} = 1, A_{j \ell} = 1\}} \sim {\sf Bernoulli}(B_{\sigma(i)\sigma(\ell)}B_{\sigma(j)\sigma(\ell)}).
\] 
So 
\[
(n - 2)B_{\#}^2 \ge \E n_{ij} = \sum_{\ell \ne i,j} B_{\sigma(i)\sigma(\ell)}B_{\sigma(j)\sigma(\ell)} \ge (n - 2) B_{\star}^2,
\] 
and we have, by the Chernoff bound, that
\begin{align}\label{eq:chernoff_nij}\nonumber
\P(n_{ij} \le (n - 2) B_{\star}^2 - \Delta) &\le \exp \left( - \frac{(\E n_{ij} - (n - 2) B_{\star}^2 + \Delta)^2}{2\E n_{ij}} \right)\\
&\le \exp \left( - \frac{\Delta^2}{2 (n - 2) B_{\#}^2} \right).
\end{align}
In order to apply Theorem~\ref{thm:pace-main} we need the following two ingredients, which we will now work out.
\begin{enumerate}
\item[(i)] estimate of $|S|$, and
\item[(ii)] estimate of $\P(N_{ij} < \tau)$.
\end{enumerate}

\noindent
\textbf{(i) Estimate of $|S|$.}
Note that $|S| = \sum_k A_{kR}$. So 
\begin{align*}
\P(|S| < m_{\star}) &= \E_R \P(|S| < m_{\star} \,|\, R). 
\end{align*}

Since $(n - 1) B_{\star} \le \E(|S| \,|\, R) = \sum_{j \ne R} B_{\sigma(k)\sigma(R)} \le (n-1) B_{\#}$, and $A_{kR}, 1 \le k \le n$ are independent, we have, by Chernoff's inequality, that
\begin{align*}
\P(|S| < m_{\star} \,|\, R) &\le \exp\left( - \frac{(\E(|S| \,|\, R) - m_{\star})^2}{2\E(|S| \,|\, R)}\right)\\
&\le\exp\left( - \frac{((n - 1)B_{\star} - m_{\star})^2}{2(n - 1) B_{\#}}\right),
\end{align*}
where $m_{\star} \le (n-1)B_{\star}.$ Therefore, the same upper bound holds for $\P(|S| < m_{\star})$. In particular, for $m_{\star} \le (n - 1)B_{\star}/2$ we have
\begin{equation}\label{eq:lower_tail_bound_on_ego_nbd_size}
\P(|S| < m_{\star}) \le \exp\left( - \frac{(n - 1)B_{\star}^2}{8B_{\#}}\right) \le \exp\left( - \frac{nB_{\star}^2}{16B_{\#}}\right).
\end{equation}

Similarly, using Chernoff's inequality for Binomial upper tail, we can show that, for $\theta > 0$,
\begin{equation}
	\P(|S| > (1 + \theta)n B_{\#}) \le \exp\left(-\frac{\theta^2 nB_{\#}}{6}\right).
\end{equation}

\noindent
\textbf{(ii) Estimate of $\P(N_{ij} < \tau)$.} Recall that $N_{ij} \,|\, n_{ij} \sim \text{Binomial}(T,n_{ij}/n)$. Then
\[
\P(N_{ij} < \tau) = \underbrace{\P\left(N_{ij} < \tau, n_{ij} < \frac{2 n\tau}{T}\right)}_{=:P_1} + \underbrace{\P\left(N_{ij} < \tau, n_{ij} \ge \frac{2n\tau}{T}\right)}_{=:P_2}.
\]
Clearly
\[
P_1 \le \P\left(n_{ij} < \frac{2n\tau}{T}\right) \le \P \left(n_{ij} < \frac{n B_{\star}^2}{2} \right). 
\]
Now, given $n_{ij}$ such that $n_{ij} \ge \frac{2n\tau}{T}$, we can invoke Chernoff's inequality to get
\[
\P(N_{ij} < \tau \,|\, n_{ij}) \le \exp\left( - \frac{(\E(N_{ij} \,|\, n_{ij}) - \tau)^2}{2\E(N_{ij} \,|\, n_{ij})}\right) \le \exp\left(- \frac{T n_{ij}}{8n}\right).
\]
Therefore
\begin{align*}
P_2 \le \E  \exp\left(- \frac{T n_{ij}}{8n}\right) \mathbf{1}_{\{n_{ij} \ge \frac{2n\tau}{T}\}} &\le \E  \exp\left(- \frac{T n_{ij}}{8n}\right)\\
&= \E  \exp\left(- \frac{T n_{ij}}{8n}\right) \mathbf{1}_{\{n_{ij} < \frac{n B_{\star}^2}{2}\}} + \E  \exp\left(- \frac{T n_{ij}}{8n}\right) \mathbf{1}_{(n_{ij} \ge \frac{n B_{\star}^2}{2})}\\
&\le \P \left(n_{ij} < \frac{n B_{\star}^2}{2}\right) + \exp\left(- \frac{T B_{\star}^2}{16}\right).
\end{align*}
Thus
\[
\P(N_{ij} < \tau) \le 2\P \left(n_{ij} < \frac{n B_{\star}^2}{2}\right) + \exp\left(- \frac{T B_{\star}^2}{16}\right).
\]
But by (\ref{eq:chernoff_nij})
\[
\P \left(n_{ij} < \frac{n B_{\star}^2}{2}\right) \le \exp\left(-\frac{((n - 2)B_{\star}^2 - \frac{n B_{\star}^2}{2})^2}{2(n - 2)B_{\#}^2}\right) \le \exp\left(-\frac{(n - 2)}{8}\frac{B_{\star}^4}{B_{\#}^2}\right) \le \exp\left(-\frac{nB_{\star}^4}{16B_{\#}^2}\right).
\]
Thus
\[
\P(N_{ij} < \tau) \le 2\exp\left(-\frac{nB_{\star}^4}{16B_{\#}^2}\right) + \exp\left(- \frac{T B_{\star}^2}{16}\right).
\]

Now we are ready to use (\ref{bound-pace}). Using our estimates on $|S|$ we get that
\[
\E |S|^2 \delta(\hat{Z}_{S}, Z_{S}) \mathbf{1}_{(|S| \ge m_{\star})} \le (1+\theta)^2 n^2 B_{\#}^2 \E \delta(\hat{Z}_{S}, Z_{S}) \mathbf{1}_{(|S| \ge m_{\star})} + \exp\left(-\frac{\theta^2 nB_{\#}}{6}\right).
\]

Next we plug into (\ref{bound-pace}) all the estimates we derived in this subsection to get the desired bound (\ref{bound2:egonet}).

\end{proof}

\subsection{Results on \gale{}}
\label{sec:gale-analysis}
For any clustering $\hat{Z_i}$ on subgraph $S_i$, let $\Pi_i \in \arg \min_{Q \text{ perm.}} \|\hat{Z_i} - Z_iQ\|_1 $, where $Z_i$ is used a shorthand for $Z_{S_i} = Z\big|_{S_i}$, the true cluster membership matrix for the members of $S_i$. Define the matrix $F_i$ by the requirement
\begin{equation}\label{eq:Zi-mc}
	\hat{Z}_i = Z_i\Pi_i+F_i,  \ \ \   F_i\in\{0,\pm 1\}^{n\times K}.
\end{equation}
In other words, $\|F_i\|_1 = \M(\hat{Z}_i, Z_i)$. 

We first analyze Algorithm~\ref{alg:match}, i.e. \match{}. Recall that, if two clusterings on some set $S$ agree, then the confusion matrix will be a diagonal matrix up to permutations, with the entries in the diagonal corresponding to the cluster sizes in either of the clusterings. In the following lemma, we consider a noisy version of this, where the two clusterings are not in perfect agreement. This lemma essentially establishes that if supplied with two clusterings whose confusion matrix is a diagonal matrix up to permutations plus noise, then \match{} will recover the correct aligning permutation, if the noise is not too large. 

\begin{lemma}
	\label{lem:match_base}
	Let $d\in \R_+^K$. Also let $M=\Pi_2^\top\diag(d)\Pi_1+\Gamma$, where $\Gamma\in \R^{K\times K}$, $\|\Gamma\|_\infty \le \min_i d_{i}/3$. Then \match{} returns $\Pi=\Pi_2^\top\Pi_1$, when applied on the confusion matrix $M$.
\end{lemma}

\begin{proof}[Proof of Lemma~\ref{lem:match_base}]
Let $D:=\diag(d)$. Let $\xi_i$ be the permutation encoded in $\Pi_i$, for $i\in[2]$, i.e. $(\Pi_i)_{uv} = \mathbf{1}_{\{ v = \xi_i(u)\}}$. It is easy to see that 
$M_{uv} = D_{\xi^{-2}_1(u), \xi^{-1}_2(v)} + \Gamma_{uv}$. We have for all $u,v$ such that $(\Pi_2^\top\Pi_1))_{uv}=0$, $M_{uv} \le \|\Gamma\|_\infty$, whereas for all $u, v$ such that $(\Pi_2^\top\Pi_1)_{uv} = 1$, $M_{uv} \ge \min_{a} D_{aa} - \|\Gamma\|_\infty$.
Hence
\[
	\min_{u, v\, :\, (\Pi_2^\top\Pi_1)_{uv} = 1} M_{uv} \ge \min_{a} D_{aa} - \|\Gamma\|_\infty \ge 2\|\Gamma\|_\infty \ge 2\max_{u, v \,:\, (\Pi_2^\top\Pi_1)_{uv} = 0} M_{uv}.
\]
Hence the top $K$ (recall that $K$ is the number of rows in $M$) elements are the diagonal elements in $D$. Thus the elements of $\Pi$ learned by the \match{} algorithm will be $\Pi_{uv} = \mathbf{1}_{\{\xi^{-1}_2(u)=\xi^{-1}_1(v)\}}$. This is equivalent to $\Pi = \Pi_2^\top \Pi_1$.
\end{proof}

Now we will establish that post-alignment misclustering errors on subgraphs equal the original misclustering errors. For this we first need a lemma on what happens when we align two subgraphs based on their intersection.
\begin{lemma}
\label{lem:induction-base-subset}
Consider two random $m$-subgraphs $S_1$ and $S_2$. Suppose our clustering algorithm $\mathcal{A}$ outputs clusterings $\hat{Z}_i, i = 1, 2$. Let $S = S_1 \cap S_2$ be of size at least $m_1 = \thresh$. Assume that $\M(\hat{Z}_i,Z_i)\le m_1\pi_{\min}/\cp$ and that \hfff{cluster_size_subgraphs}$ \min_{k \in [K]} n_k^{(S)} \ge m_1\pi_{\min}$. Then $\|\hat{Z}_2\Pi-Z\Pi_1\|=\M(\hat{Z}_2,Z)$, where 
$\Pi$ is the output of $\match{}(\hat{Z}_2\big|_{S}, \hat{Z}_1\big|_S)$.
\end{lemma}

\begin{proof}[Proof of Lemma~\ref{lem:induction-base-subset}]
To ease notation, let us write $\tilde{Z}_i := \hat{Z}_i\big|_S$, $\tilde{F}_i := F_i\big|_S, i =1, 2$, and $\tilde{Z} := Z\big|_S$. Then by restricting \eqref{eq:Zi-mc} to $S$, we get
\[
\tilde{Z}_i = \tilde{Z}\Pi_i+ \tilde{F}_i,  \ \ \   \tilde{F}_i\in\{0,\pm 1\}^{n\times K}.
\]
Now
\[
	\tilde{Z}_2^\top \tilde{Z}_1 = \Pi_2^\top \tilde{Z}^\top \tilde{Z}\Pi_1 + \underbrace{\Pi_2^\top \tilde{Z}^\top \tilde{F}_1 + \tilde{F}_2^\top \tilde{Z} \Pi_1 + \tilde{F}_2^\top \tilde{F}_1}_{=: \Gamma}.
\]
Note that (a) multiplication by a permutation matrix does not change the $\|\cdot\|_1$ norm, and (b) for any matrix $A$, we have $\|A\|_1 = \|A^{\top}\|_1$. Therefore
\[
\|\Pi_2^\top \tilde{Z}^\top \tilde{F}_1\|_1 \overset{(\text{by (a)})}{=} \|\tilde{Z}^\top \tilde{F}_1\|_1 = \sum_{a, b} |\sum_i(\tilde{Z}^\top)_{ai} (\tilde{F}_1)_{ib}| \le \sum_{a, b} \sum_i \tilde{Z}_{ia} |(\tilde{F}_1)_{ib}| = \sum_{i, b} |(\tilde{F}_1)_{ib}| \times \underbrace{\sum_a \tilde{Z}_{ia}}_{= 1}  = \|\tilde{F}_1\|_1.
\]
Similarly,
\[
\| \tilde{F}_2^\top \tilde{Z} \Pi_1 \|_1 \overset{(\text{by (b)})}{=} \| \Pi_1^\top \tilde{Z}^\top \tilde{F}_2 \|_1 \le \|\tilde{F}_2\|_1.
\]
Finally,
\[
\|\tilde{F}_2^\top \tilde{F}_1\|_1 \le \sum_{a, b} \sum_i |(\tilde{F}_2)_{ia} (\tilde{F}_1)_{ib}| = \sum_{i, b} |(\tilde{F}_1)_{ib}| \times \underbrace{\sum_a |(\tilde{F}_2)_{ia})|}_{= 2} =  2\|\tilde{F}_1\|_1,
\]
where we have used the fact that each row of $\tilde{F}_i$ has exactly one $1$ and one $-1$. By (b), $\|\tilde{F}_2^\top \tilde{F}_1\|_1 = \|\tilde{F}_1^\top \tilde{F}_2\|_1 \le 2 \|\tilde{F}_2\|_1$, and therefore
\[
	\|\tilde{F}_2^\top \tilde{F}_1\|_1 \le \|\tilde{F}_1\|_1 + \|\tilde{F}_2\|_1.
\]
Note that, by our assumptions on the individual misclustering errors, 
\[
	\|\tilde{F}_i\|_1 \le \|F_i\|_1 = \M(\hat{Z}_i,Z_i)\le m_1\pi_{\min}/\cp.
\]
Therefore
\[
	\|\Gamma\| \le 2(\|\tilde{F}_1\|_1 + \|\tilde{F}_2\|_1) \le m_1\pi_{\min}/3.
\]
Since by our assumption, $(\tilde{Z}^\top \tilde{Z})_{kk} = n_{k}^{(S)} \ge m_1\pi_{\min}$, we can apply Lemma~\ref{lem:match_base} to see that the output $\match{}(\tilde{Z}_2,\tilde{Z}_1)$ is $\Pi=\Pi_2^\top \Pi_1$. Hence 
\[
	\|\hat{Z}_2\Pi - Z\Pi_1\|_1 = \|\hat{Z}_2\Pi_2^\top \Pi_1 - Z\Pi_1\|_1
	=\|\hat{Z}_2\Pi_2^\top - Z\|_1=\|\hat{Z}_2 - Z\Pi_2\|_1=\|F_2\|_1.
\]
\end{proof}

\begin{prop}\label{prop:alignerr}
Let the $T$ subsets, $(S_i)_{i=1}^T$ be associated with estimated clusterings $\hat{Z}_i$ with misclustering error $\M (\hat{Z}_i, Z_i) \le m_1\pi_{\min}/\cp$. Consider a traversal $x_1 \sim x_2 \sim \cdots \sim x_J$ of some spanning tree $\mathcal{T}$ of $\mathcal{S}_{m ,T}$ as in Algorithm~\ref{alg:seq}, satisfying
\[
\min_{2 \le j \le J} \min_{k \in [K]} n_k^{(S_{x_i} \cap S_{x_{i - 1}})} \ge m_1 \pi_{\min},
\]
where $n_k^{(S)}$ is the number of nodes from cluster $k$ in a subgraph $S$. Applying \gale{} on this walk: let $\hat{Z}^{(x_1)} := \hat{Z}_{x_1}, \hat{\Pi}_{x_1} = I$ and for $2 \le i \le J$, define recursively
\[
\hat{\Pi}_{x_i} = \begin{cases}I & \text{if $x_i$ was visited before, i.e. $x_i = x_j$ for some $1 \le j < i$},\\
\match{}(\hat{Z}_{x_i}\big|_{S_{x_i} \cap S_{x_{i - 1}}}, \hat{Z}^{(x_{i - 1})}\big|_{S_{x_i} \cap S_{x_{i - 1}}}) & \text{otherwise},
\end{cases}
\]
and set $\hat{Z}^{(x_i)} = \hat{Z}_{x_i} \hat{\Pi}_{x_i}.$ Then, for any $1 \le i \le J$, we have that
\[
	\M(\hat{Z}^{(x_i)}, Z_{x_i}) = \M(\hat{Z}_{x_i}, Z_{x_i}).
\]
\end{prop}

\begin{proof}
We claim that, for all $1 \le i \le J$, we have 
\begin{equation}\label{eq:claim}
\Pi_{x_1} \in \arg\min_{Q} \|\hat{Z}^{(x_i)} - Z_{x_i}Q\|, \text{ and } \M(\hat{Z}^{(x_i)}, Z_{x_i}) = \M(\hat{Z}_{x_i}, Z_{x_i}).
\end{equation} 
We use strong induction to prove this claim. \eqref{eq:claim} is true by definition, for $i = 1$. Now assume that it is true for all $i \le \ell$. Then, for all $i \le \ell$, we have the representation $\hat{Z}^{(x_i)} = Z_{x_i} \Pi_{x_1} + \bar{F}_{x_i}$, for some matrix $\bar{F}_{x_i} \in \{0, \pm 1\}^{m \times K}$. If $x_{\ell + 1}$ has been visited before, i.e. $x_{\ell + 1} = x_j$ for some $j \le \ell$, then \eqref{eq:claim} holds by our induction hypothesis. Otherwise, we can apply Lemma~\ref{lem:induction-base-subset} on the two clusterings $\hat{Z}^{(x_\ell)}$ on $S_{x_\ell}$ and $\hat{Z}_{x_{\ell + 1}}$ on $S_{x_{\ell + 1}}$, to conclude that
\[
\|\hat{Z}^{(x_{\ell + 1})} - Z_{x_{\ell + 1}} \Pi_{x_1}\|_1 = \|\hat{Z}_{x_{\ell + 1}}\hat{\Pi}_{x_{\ell + 1}} - Z_{x_{\ell + 1}} \Pi_{x_1}\|_1 = \mathcal{M}(\hat{Z}_{x_{\ell + 1}}, Z_{x_{\ell + 1}}).
\]
But $\hat{Z}^{(x_{\ell + 1})} = \hat{Z}_{x{\ell + 1}}\hat{\Pi}_{x_{\ell + 1}}$, and $\hat{\Pi}_{x_{\ell + 1}} = \Pi_{x_{\ell+1}}^\top \Pi_{x_1}$. So, for any permutation matrix $Q$ we have
\begin{align*}
\|\hat{Z}^{(x_{\ell + 1})} - Z_{x_{\ell + 1}} Q\|_1 & = \|\hat{Z}_{x_{\ell + 1}}\hat{\Pi}_{x_{\ell + 1}} - Z_{x_{\ell + 1}} Q\|_1\\
&= \|\hat{Z}_{x{\ell + 1}} - Z_{x_{\ell + 1}} Q \hat{\Pi}_{x_{\ell + 1}}^\top\|_1\\
&\ge \mathcal{M}(\hat{Z}_{x_{\ell + 1}}, Z_{x_{\ell + 1}}) = \|\hat{Z}^{(x_{\ell + 1})} - Z_{x_{\ell + 1}} \Pi_1\|_1.
\end{align*}
Thus, indeed $\Pi_{x_1} \in \arg\min_{Q} \|\hat{Z}^{(x_{\ell + 1})} - Z_{x_\ell}Q\|$ and as a consequence  $\M(\hat{Z}^{(x_{\ell + 1})}, Z_{x_{\ell + 1}}) = \M(\hat{Z}_{x_{\ell + 1}}, Z_{x_{\ell + 1}}).$ The induction is now complete.
\end{proof}

Now we establish an upper bound on the error of the estimator $\hat{Z}^{\gale{}}$ in terms of the errors of the aligned clusterings $\hat{Z}^{(\ell)}$.
\begin{prop}\label{prop:err_algseq}
Let $0 < \theta < 1$, and set $\tau = \frac{\theta Tm}{n}.$ Then, for any cluster membership matrix $\bar{Z} \in \{0, 1\}^{n \times K}$,

\begin{equation}\label{eq:bound_main_gale}
	\frac{\|\hat{Z}^{\gale{}}-\bar{Z}\|_F^2}{n} \le \frac{1}{\theta T}\sum_{\ell = 1}^T \frac{\|\hat{Z}^{(\ell)} - \bar{Z}_{\ell}\|_F^2}{m} + \sum_i \mathbf{1}_{\{N_i < \tau \}}.
\end{equation}
\end{prop}

\begin{proof}
	The proof is very similar to the proof of Theorem~\ref{thm:pace-main}. We decompose
	\[
	\bar{Z}_{ik} = \bar{Z}_{ik}\mathbf{1}_{\{N_i \ge \tau\}} + \bar{Z}_{ik}\mathbf{1}_{\{N_i < \tau\}}.
	\] 
	Using this we have
	\[
	\hat{Z}^{\gale{}}_{ik} - \bar{Z}_{ik} = \frac{\sum_{\ell = 1}^T y^{(\ell)}_i (\hat{Z}^{(\ell)}_{ik} - \bar{Z}_{ik})}{N_i}\mathbf{1}_{\{N_i \ge \tau\}} - \bar{Z}_{ik}\mathbf{1}_{\{N_i < \tau\}}.
	\]
	Therefore
	\begin{align*}
		\|\hat{Z}^{\gale{}} - \bar{Z}\|^2=\sum_{i, k} ( \hat{Z}^{\gale{}}_{ik} - \bar{Z}_{ik})^2 &= \sum_{i, k} \left(\frac{\sum_{\ell = 1}^T y^{(\ell)}_i (\hat{Z}^{(\ell)}_{ik} - \bar{Z}_{ik})}{N_i}\right)^2\mathbf{1}_{\{N_i \ge \tau\}} + \sum_{i, k} \bar{Z}_{ik}^2 \mathbf{1}_{\{N_i < \tau\}}\\
		&\le \underbrace{\sum_{i, k} \left(\frac{\sum_{\ell = 1}^T y^{(\ell)}_i (\hat{Z}^{(\ell)}_{ik} - \bar{Z}_{ik})}{N_i}\right)^2\mathbf{1}_{\{N_i \ge \tau\}}}_{=: E} + \sum_{i} \mathbf{1}_{\{N_i < \tau\}}.
	\end{align*} 

	We now analyze $E$. Let $W_i:=\mathbf{1}_{(N_i\ge \tau)}/N_i$. Noting that $E = \sum_{i, k} W_i^2 (\sum_{\ell = 1}^T y^{(\ell)}_i (\hat{Z}^{(\ell)}_{ik} - \bar{Z}_{ik}))^2 $, an application of Cauchy-Schwartz inequality gives
	\begin{align}\label{eq:seq-zdev}
		\nonumber
		E \le \sum_{i, k} W_i^2 \underbrace{\left(\sum_{\ell = 1}^T y^{(\ell)}_i\right)}_{= N_i} \sum_{\ell = 1}^T y^{(\ell)}_i (\hat{Z}^{(\ell)}_{ik} - \bar{Z}_{ik})^2 &=    \sum_{\ell = 1}^T\sum_{i, k} W_i y^{(\ell)}_i(\hat{Z}^{(\ell)}_{ik} - \bar{Z}_{ik})^2 \nonumber\\ 
		&\le \max_{i} W_i \sum_{\ell = 1}^T \sum_{i, k} y^{(\ell)}_i(\hat{Z}^{(\ell)}_{ik} - \bar{Z}_{ik})^2\nonumber\\
		&\le \max_i W_i \sum_{\ell = 1}^T \|\hat{Z}^{(\ell)} - \bar{Z}_{\ell}\|_F^2.
	\end{align}
	Note that $W_i = \frac{\mathbf{1}_{(N_i\ge \tau)}}{N_i} \le \frac{1}{\tau} = \frac{n}{\theta T m}$. Therefore
	\begin{equation}\label{eq:bound_E}
		E \le \frac{n}{\theta T} \sum_{\ell = 1}^T \frac{\|\hat{Z}^{(\ell)} - \bar{Z}_{\ell}\|_F^2}{m}.
	\end{equation}
\end{proof}

We now mention some auxiliary results, whose proofs are deferred to the Appendix. These will help us control the probability of a ``bad'' event, defined in the proof of Theorem~\ref{thm:gale-main}.

\begin{lemma}\label{lem:nicount}
	Let $\tau = \frac{\theta T m}{n}$, where $0 < \theta < 1$. Let $r > 0$. Then, with probability at least $1 - \frac{1}{n^r}$, we have
	\[
	\sum_{i}\mathbf{1}_{\{ N_i < \tau\}} \le ne^{-(1 - \theta)^2 Tm/2n} \left(1 + \sqrt{\frac{3r \log n}{ne^{-(1 - \theta)^2 Tm/2n}}}\right).
	\]
\end{lemma}

Recall that we view the $T$ random $m$-subsets of $[n]$ as nodes of a super-graph $\mathcal{S}_{m, T}$ and put an edge between nodes $a (\equiv S_a)$ and $b (\equiv S_b)$ in $\mathcal{S}_{m, T}$ (we use the shorthand $a \sim b$ to denote an edge between $a$ and $b$), if \hfff{intersection_size}$Y_{ab}:=|S_a \cap S_b| \ge m_1 = \thresh$. The next lemma shows that $\mathcal{S}_{m, T}$ is in fact an Erd\"{o}s-R\'{e}nyi random graph.

\begin{lemma}\label{lem:erygraph}
The super-graph $\mathcal{S}_{m, T}$ is an Erd\"{o}s-R\'{e}nyi random graph with
\begin{equation}
	P(a\sim b) = P\left(Y_{ab} \ge \thresh\right) \ge 1 - \exp\left(-\frac{m^2}{16n}\right).
\end{equation}
\end{lemma}
Because of our assumptions on $m$, it follows that $S_{m, T}$ is well above the connectivity threshold for Erd\"{o}s-R\'{e}nyi random graphs. 
\begin{lemma}\label{lem:erygraph_connected}
The super-graph $\mathcal{S}_{m, T}$ is connected with probability at least $1 - \exp(-O(n)).$
\end{lemma}

The next lemma states that intersection of two random $m$-subgraphs contains enough representatives from each cluster, with high probability. 
\begin{lemma}\label{lem:cluster_presence}
Consider two random $m$-subgraphs $S_a$ and $S_b$. Let $r > 0$ and $\frac{m^2 \pi_k}{n \log n} \ge 20r$. Then
\[
P\left(n_k^{(S_a \cap S_b)} < \frac{m^2\pi_k}{n}\left(1-O\left(\sqrt{\frac{2rn\log n}{m^2\pi_k}}\right)\right)\right) \le \frac{2}{n^r}.
\]
\end{lemma}

We are now ready to prove our main result on \gale{}.
\begin{proof}[Proof of Theorem~\ref{thm:gale-main}]
First we construct a good set in the sample space. Consider the following ``bad'' events
\begin{align*}
\mathcal{B}_{1} &:= \{(S_1,\ldots,S_T) \mid \mathcal{S}_{m, T} \text{ is connected}\},\\
\mathcal{B}_{2} &:= \left\{(S_1,\ldots,S_T) \mid \min_{(i,j)} \min_k n_k^{(S_i\cap S_j)} < m_1\pi_{\min}\right\},\\
\mathcal{B}_{3} &:= \left\{(S_1,\ldots,S_T) \mid \sum_{i} \mathbf{1}_{\{N_i < \tau\}} > e^{-(1 - \theta)^2 Tm/2n} \left(1 + \sqrt{\frac{3r' \log n}{ne^{-(1 - \theta)^2 Tm/2n}}}\right) \right\},\\
\mathcal{B}_4 &:= \{(A, S_1, \ldots, S_T) \mid \max_{1 \le i \le T}\M(\hat{Z}_{i},Z_i) > m_1\pi_{\min}/\cp\}.
\end{align*}

Let $\mathcal{B}:= \cup_{i = 1}^4 \mathcal{B}_i$. Let $r''> 0$ By Lemma~\ref{lem:erygraph_connected}, we have $\P(\mathcal{B}_1) \le \exp(-O(n))$. Lemma~\ref{lem:cluster_presence}, and a union bound gives $\P(\mathcal{B}_2) \le \binom{T}{2} \times K \times \frac{2}{n^{r''}} \le \frac{KT^2}{n^{r''}}$. By Lemma~\ref{lem:nicount}, $\P(\mathcal{B}_3) \le \frac{1}{n^{r'}}$. Finally, by our hypothesis on individual misclustering errors, we have, by a union bound, that $\P(\mathcal{B}_4) \le T\delta$. Therefore
\[
\P(\mathcal{B}) \le \sum_{i = 1}^4 \P(\mathcal{B}_i) \le \exp(-O(n)) + \frac{KT^2}{n^{r''}} + \frac{1}{n^{r'}} + T\delta \le T\delta + O\left(\frac{1}{n^{r'}}\right),
\]
by choosing $r''$ suitably large.

Now on the good set $\mathcal{B}^c$, for any $\mathcal{T} \in {\sf SpanningTrees}_{\mathcal{S}_{m, T}}$ and for any $(x_1, \ldots, x_J) \in {\sf Traversals}_{\mathcal{T}}$, the hypothesis of Propositions~\ref{prop:alignerr} and \ref{prop:err_algseq} are satisfied. Now note that
\[
\delta(\hat{Z}^{\gale}_{\mathcal{T},(x_1, \ldots, x_J)}, Z) \le \frac{\|\hat{Z}^{\gale}_{\mathcal{T},(x_1, \ldots, x_J)}-Z\Pi_{x_1}\|_1}{n} = \frac{\|\hat{Z}^{\gale}_{\mathcal{T},(x_1, \ldots, x_J)} - Z\Pi_{x_1}\|_F^2}{n}.
\]
By Proposition~\ref{prop:alignerr}, we also have, for any $1 \le \ell \le T$, that
\[
\frac{\|\hat{Z}^{(\ell)} - Z_{\ell}\Pi_{x_1}\|_F^2}{m} = \frac{\|\hat{Z}^{(\ell)} - Z_{\ell}\Pi_{x_1}\|_1}{m} = \delta(\hat{Z}^{(\ell)}, Z) = \delta(\hat{Z}_{\ell}, Z).
\] 
Thus, by of Proposition~\ref{prop:err_algseq}, taking $\bar{Z}=Z\Pi_{x_1}$ and noting then that $\bar{Z}_{\ell} = Z_{\ell}\Pi_{x_1}$, we have
\[
\delta(\hat{Z}^{\gale}_{\mathcal{T},(x_1, \ldots, x_J)}, Z) \le \frac{1}{\theta T} \sum_{\ell = 1}^T \delta(\hat{Z}_\ell,Z) + \sum_i \mathbf{1}_{\{N_i < \tau \}}.
\]
Since the bound on the RHS does not depend on the particular spanning tree used, or a particular traversal thereof, we conclude that, on the good event $\mathcal{B}^c$, 

\[
\max_{\mathcal{T} \in {\sf SpanningTrees}_{\mathcal{S}_{m, T}}} \max_{(x_1, \ldots, x_J) \in {\sf Traversals}_{\mathcal{T}}} \delta(\hat{Z}^{\gale}_{\mathcal{T},(x_1, \ldots, x_J)}, Z) \le \frac{1}{\theta T} \sum_{\ell = 1}^T \delta(\hat{Z}_\ell,Z) + \sum_i \mathbf{1}_{\{N_i < \tau \}}.
\]
Now, for our choice of $\tau$, and $T \ge \frac{2r''' n \log n}{(1 - \theta)^2 m}$, we have $e^{-(1 - \theta)^2 Tm/2n} \le \frac{1}{n^{r'''}}$. Therefore, on the event $\mathcal{B}^c$ 
\[
\sum_i \mathbf{1}_{\{N_i < \tau \}} \le e^{-(1 - \theta)^2 Tm/2n} \left(1 + \sqrt{\frac{3r' \log n}{ne^{-(1 - \theta)^2 Tm/2n}}}\right) = O\left(\sqrt{\frac{\log n}{n^{1+r'''}}}\right) = O\left(\frac{1}{n^r}\right),
\]
by choosing $r'''$ suitably large.

Thus we conclude that with probability at least $1 - T\delta - O\left(\frac{1}{n^{r'}}\right)$, we have
\[
\max_{\mathcal{T} \in {\sf SpanningTrees}_{\mathcal{S}_{m, T}}} \max_{(x_1, \ldots, x_J) \in {\sf Traversals}_{\mathcal{T}}} \delta(\hat{Z}^{\gale}_{\mathcal{T},(x_1, \ldots, x_J)}, Z) \le \frac{1}{\theta T} \sum_{\ell = 1}^T \delta(\hat{Z}_\ell,Z) + O\left(\frac{1}{n^r}\right).
\]
\end{proof}

%%%%%%%%%%%%%%%%%%%%%%%%%%%%%%%%%%%%%%%%%%%%%%%%%%%%%%%%
\section{Acknowledgments}\label{sec:acknowledgements}
We thank Arash Amini, David Blei, David Choi for their valuable comments when versions of this work were presented at conferences. SSM thanks Aditya Guntuboyina, Alan Hammond, Luca Trevisan and Bin Yu for their valuable comments during his quals talk at Berkeley, which was based on this paper.

%%%%%%%%%%%%%%%%%%%%%%%%%%%%%%%%%%%%%%%%%%%%%%%%%%%%%%%%
\bibliographystyle{apalike}
\bibliography{local_clustering.bib}

%%%%%%%%%%%%%%%%%%%%%%%%%%%%%%%%%%%%%%%%%%%%%%%%%%%%%%%%
\newpage
\appendix
%%%%%%%%%%%%%%%%%%%%%%%%%%%%%%%%%%%%%%%%%%%%%%%%%%%%%%%%
\section{Details of \dgcluster{}}\label{sec:app_greedy}
Here we detail our distance based greedy algorithm \dgcluster{}. The idea behind \dgcluster{} is to note that, if $i$ and $j$ are in the same community, then $\|C_{i\star} - C_{j\star}\| = 0$, and otherwise $\|C_{i\star} - C_{j\star}\| = \sqrt{|\sigma(i)| + |\sigma(j)|} = \Theta (\sqrt{2n/K})$ (when the communities are balanced). Thus, we expect to be able to cluster the vertices using $d_{ij} = \|\hat{C}_{proj, i\star} - \hat{C}_{proj, j\star}\|$, namely by starting with a root vertex $r_1$ and for some threshold $\gamma$ putting all vertices $j$ satisfying $d_{r_1j} \le \gamma$ in the same cluster as $r_1$, and then picking another root vertex $r_2$ from the remaining set, and putting all vertices $j$ in the remaining set that satisfy $d_{r_2j} \le \gamma$ in the same cluster as $r_2$, and so on. Here a root vertex $r_i$ may be chosen as one of the vertices with the highest degree in the remaining set (or according to a degree-weighted random sampling scheme), to give importance to highly connected vertices. We also note that depending on the threshold $\gamma$ the number of blocks we get can vary. In practice, we will start with small $\gamma$ (yielding a large number of communities), and stop at the smallest $\gamma$ that gives us $\ge K$ blocks. If we get more than $K$ blocks, we merge, in succession, pairs of blocks having the largest intersection in $\hat{C}$ (relative to their sizes) until we get exactly $K$ blocks. A rule of thumb would be to start with $\gamma = c\sqrt{2n/K}$ with $c$ small and then gradually increase $c$.

\begin{algorithm}[H]
	\caption{A distance based greedy clustering algorithm: \dgcluster{}}
	\label{alg:dgcluster}
	\begin{algorithmic}[1]
		\State Input: $C$, $C'$, $K$. Output: $\sigma = \dgcluster{}(C, C', K)$, a clustering based on distances between the rows of $C'$, with merging guidance from $C$, if necessary.
		\State Set $\theta = \sqrt{2n/K}$. $c = c_{min} = 0, \delta = 0.01.$
		\State Set $K_{now} = n$;
		\While {flag = TRUE}
		\State $c \gets c + \delta$.
		\State $\sigma_{temp} \gets \naivecluster{}(C', c\theta).$
		\State $K_{now} = \max_i \sigma_{temp}(i)$.
		\If {$K_{now} \ge K$}
		\State flag $\gets$ TRUE.
		\State $\sigma \gets \sigma_{temp}.$
		\Else
		\State	flag $\gets$ FALSE.
		\EndIf
		\EndWhile
		\State $K_{now} \gets \max_i \sigma(i).$
		\While {$K_{now} > K$}
		\State Let $\Gamma_i = \sigma^{-1}(\{i\}), i = 1, \ldots, K_{now}$.
		\State Compute the (upper triangle of the) matrix $R$, where
		\[
		R_{ij} \gets \frac{\sum_{k \in \Gamma_i, l \in \Gamma_j} C_{kl}}{|\Gamma_i| |\Gamma_j|}.
		\]
		\State Pick $(i_\star, j_\star) \in \arg \max_{i, j} R_{ij}$.
		\State $\sigma \gets \merge{}(\sigma, i_\star, j_\star).$
		\EndWhile 
	\end{algorithmic}
\end{algorithm}
Algorithm~\ref{alg:dgcluster} in turn makes use of the following two algorithms.
\begin{algorithm}[H]
	\caption{A naive clustering algorithm: \naivecluster{}}
	\label{alg:naive_cluster}
	\begin{algorithmic}[1]
		\State Input: $C$, $\gamma$ : a threshold. Output: $\sigma =\naivecluster{}(C, \gamma)$, a clustering based on distances between the rows of $C$.
		\State Set \texttt{Unassigned} $= \{1,\ldots,n\}$, $b = 1$, $\sigma \equiv 0.$\vskip5pt
		\While {\texttt{Unassigned} $\ne \emptyset$}
		\State $i \gets$ a random (uniform or degree-weighted etc.) index in \texttt{Unassigned}.
		\State $\sigma(i) \gets b$.
		\State \texttt{Unassigned} $\gets \texttt{Unassigned}\setminus\{i\}.$
		\For {$j \in \texttt{Unassigned}$}
		\State Compute $d_{ij} = \|\hat{C}_{i\star} - \hat{C}_{j\star}\|$
		\If {$d_{ij} \le \gamma$}
		\State $\sigma(j) \gets b$
		\State \texttt{Unassigned} $\gets \texttt{Unassigned}\setminus\{j\}.$
		\EndIf
		\EndFor
		\State $b\gets b + 1.$
		\EndWhile
	\end{algorithmic}
\end{algorithm}

\begin{algorithm}[H]
	\caption{\merge{}}
	\label{alg:merge}
	\begin{algorithmic}[1]
		\State Input: $\sigma, a, b$. Output: $\sigma = \merge{}(\sigma, a, b)$, a clustering with blocks $a$ and $b$ merged 
		\State $u \gets \min \{ a, b \}, v \gets \max \{ a, b \}$
		\For {$i = 1, \ldots, n$}
		\If {$\sigma(i) = v$}
		\State $\sigma(i) \gets u$.
		\ElsIf {$\sigma(i) > v$} 
		\State $\sigma(i) \gets \sigma(i) - 1$.
		\EndIf
		\EndFor
	\end{algorithmic}
\end{algorithm}

%%%%%%%%%%%%%%%%%%%%%%%%%%%%%%%%%%%%%%%%%%%%%%%%%%%%%%%%
\section{Some auxiliary results}\label{sec:app_aux_results}
\begin{lemma}[Thresholding \pace{}]\label{lem:consistency_thresholded_mx}
Let $\hat{C}_{\eta} := [\hat{C} > \eta]$. We have
\begin{equation}
\E\tilde{\delta}(\hat{C}_{\eta},C) \le \frac{\E\tilde{\delta}(\hat{C}_{\eta},C)}{\min\{\eta^2, (1 - \eta)^2\}}.
\end{equation}
In particular, for $\eta = 1/2$, we have
\[
	\E\tilde{\delta}(\hat{C}_{1/2},C) \le 4\E\tilde{\delta}(\hat{C},C).
\]
\end{lemma}
\begin{proof}
Note that $|(\hat{C}_{\eta})_{ij} - C_{ij}| \sim$ Ber$(q_{ij})$, where
\[
	q_{ij} = C_{ij} \P(\hat{C}_{ij} \le \eta) + (1 - C_{ij}) \P(\hat{C}_{ij} > \eta).
\]
Thus $\E ((\hat{C}_{\eta})_{ij} - C_{ij})^2 = q_{ij}$. Now
\[
	\E(\hat{C}_{ij} - C_{ij})^2 = \E(\hat{C}_{ij} - C_{ij})^2 \mathbf{1}_{\{ \hat{C}_{ij} \le \eta \}} + \E(\hat{C}_{ij} - C_{ij})^2 \mathbf{1}_{\{ \hat{C}_{ij} \le \eta \}}.
\]
Therefore
\[
\E(\hat{C}_{ij} - C_{ij})^2 \ge \begin{cases}
	(1 - \eta)^2 \P(\hat{C}_{ij} \le \eta) = (1 - \eta)^2 q_{ij}, &\text{ if } C_{ij} = 1\\
	\eta^2 \P(\hat{C}_{ij} > \eta) = \eta^2 q_{ij}, &\text{ if } C_{ij} = 0.
\end{cases}
\]
This means
\[
	\E ((\hat{C}_{\eta})_{ij} - C_{ij})^2 \le \frac{\E(\hat{C}_{ij} - C_{ij})^2}{\min\{\eta^2, (1 - \eta)^2\}},
\]
which, on summing over $i,j$, gives us the desired bound.
\end{proof}

\begin{lemma}[Rounding \gale{}]\label{lem:rounding}
	Consider an estimated cluster membership matrix $\hat{Z}\in [0,1]^{n\times K}$ and the true cluster membership matrix $Z\in\{0,1\}^{n\times K}$. Then $\|\round(\hat{Z})-Z\|_1\leq 2\|\hat{Z}-Z\|_1$.
\end{lemma}

\begin{proof}
	Let $S:=\{(i,k)\in[n]\times [K]:|\hat{Z}_{ik}-Z_{ik}|\ge 1/2\}$. Then by Markov's inequality, we have $|S|\le 2\|\hat{Z} - Z\|_1$. Note that for $(i,k)\in[n]\times [K]\setminus S$, $\round(\hat{Z}_{ik})=Z_{ik}$. Thus, 
	\[
		\|\round(\hat{Z})-Z\|_1	= \sum_{(i,k)\in S} |\round(\hat{Z}_{ik})-Z_{ik}|\leq |S|\le 2 \|\hat{Z} - Z\|_1
	\]
	\end{proof}
	We shall need the following Bernstein-type concentration result for hypergeometric random variables.
\begin{lemma}[Corollary~$1$ of \cite{greene2017} (restated here using slightly different notation)]\label{lem:hyper}
	Consider a random variable $H \sim {\sf Hypergeometric}(k, \ell, L)$. Let $\mu = \ell/L$, $\sigma^2 = \mu (1 - \mu)$, $f = (k-1)/(L-1)$. Then for $\lambda>0$, 
	\[
		\P\left( \sqrt{k}\left(\frac{H}{k} - \mu \right) > \lambda \right)	\le \exp \left(- \frac{\lambda^2/2}{\sigma^2 (1 - f) + \frac{\lambda}{3 \sqrt{k}}} \right).
	\]
\end{lemma}
Let us deduce from this two convenient Chernoff type bounds on the upper and lower tail of a hypergeometric variable. 
\begin{lemma}\label{lem:hyper_lower}
	Consider a random variable $H \sim {\sf Hypergeometric}(k, \ell, L)$. Then for $0 < \epsilon < 1$, we have
	\[
		\max\{\P\left( H < (1 - \epsilon)\E H\right), \P\left( H > (1 + \epsilon)\E H\right)\} \le \exp \left(- \frac{\epsilon^2 \E H}{2(1 + \epsilon/3)} \right) \le \exp \left(- \frac{\epsilon^2 \E H}{4} \right).
	\]
\end{lemma}

\begin{proof}
Note that $k - H \sim {\sf Hypergeometric}(k, L - \ell, L)$. Using Lemma~\ref{lem:hyper} we get
\[
	\P\left(\frac{H}{k}-\frac{\ell}{L} < - \frac{\lambda}{\sqrt{k}}\right) = \P \left(\sqrt{k}\left(\frac{k - H}{k} - \frac{L - \ell}{L}\right) > \lambda\right) \le \exp \left(-\frac{\lambda^2/2}{\frac{\ell}{L}+\frac{\lambda}{3\sqrt{k}}}\right).
\]
Taking $\frac{\lambda}{\sqrt{k}} = \frac{\ell}{L}\epsilon$, we get
	\[
		\P\left(H < \frac{k\ell}{L} (1 - \epsilon)\right) \le \exp\left( - \frac{\ell^3 k\epsilon^2/(2L^2)}{(1 + \epsilon/3)\ell/(L)}\right)=\exp\left(-\frac{\epsilon^2k\ell}{2(1 + \epsilon/3)L}\right).
	\]
	This gives us the desired bound for the lower tail, the upper tail can be handled similarly.
\end{proof}

\subsection{Analysis of \pace{} under stochastic block model}
Recall that we need to know how an algorithm performs on a randomly selected subgraph under our subgraph selection procedure. Here we will discuss how one can obtain such guarantees under the stochastic blockmodel. This will require us to understand the behavior of sizes of different communities in a (randomly) chosen subgraph.

\subsubsection{Community sizes in random \texorpdfstring{$m$}{Lg}-subgraphs}
(The results of this subsection do not depend on any modeling assumption.) Let $S$ be a random $m$-subgraph, and set $m/n = q$. Then if $(n_1^{(S)},\ldots,n_K^{(S)})$ is the cluster size vector for $Z_S$, then clearly $n^{(S)}_k \sim$ {\sf Hypergeometric}$(m, n_k, n)$ and therefore, by Lemma~\ref{lem:hyper_lower},
\begin{align*}
\P(n^{(S)}_k \le n_{\min} q - \Delta) &\le \exp\left(-\frac{((n_k - n_{\min})q+\Delta)^2}{4 n_k q}\right)\\
&\le \exp\left(-\frac{\Delta^2}{4 n_{\max} q}\right).
\end{align*}
Therefore, by union bound,
\begin{align*}
	\P(n^{(S)}_{\min} \le n_{\min} q - \Delta)&\le \sum_k \P(n^{(S)}_k \le n_{\min} q - \Delta)\\ 
	& \le K \exp\left(-\frac{\Delta^2}{4 n_{\max} q}\right).
\end{align*}
Choosing $\Delta = \sqrt{4r' n_{\max} q \log n}$, where $r' > 0$, we see that with probability at least $1 - \frac{K}{n^{r'}}$ we have
\begin{equation}\label{eq:min_comm_random}
n^{(S)}_{\min} \ge n_{\min} q - \sqrt{4r' n_{\max} q \log n}
\end{equation}
Similarly, we can show that
\[
\P(n_{\max}^{(S)} \ge n_{\max} q  + \Delta) \le K\exp\left( - \frac{\Delta^2}{4 n_{\max} q} \right),
\] 
and then, taking $\Delta = \sqrt{4r'n_{\max}q\log n}$, conclude that with probability at least $1 - \frac{K}{n^{r'}}$
\begin{equation}\label{eq:max_comm_random}
n_{\max}^{(S)} \le n_{\max} q  + \sqrt{4r'n_{\max}q\log n}.
\end{equation}
	
\subsubsection{Community sizes in \texorpdfstring{$1$}{Lg}-hop ego neighborhoods}
Let now $S$ be a randomly chosen $1$-hop ego neighborhood. Note that the size of the $k$-th block in this neighborhood satisfies
\[
n_{k}^{(S)} = \sum_{j \,:\, \sigma(j) = k} X_j,
\]
where $X_j = \mathbf{1}_{\{j \in S\}}$. 

Now it is not hard to see that the $X_j$'s are independent conditional $R$, the root of $S$. We also have that $\E(X_j \,|\, R) = B_{\sigma(j)\sigma(R)}\mathbf{1}(j\ne R).$ It follows that
\[
\E(n_k^{(S)} \,|\, R) = \sum_{j \,:\, \sigma(j) = k} \E(X_j \,|\, R) = \sum_{j \,:\, \sigma(j) = k} B_{\sigma(j)\sigma(R)}\mathbf{1}_{\{j\ne R\}},
\] 
which means that
\[
(n_{\min} - 1) B_* \le (n_k - 1) B_{\star} \le \E(n_k^{(S)} \,|\, R) \le n_k B_{\#} \le n_{\max} B_{\#}.
\]
Therefore, by Chernoff's inequality,
\begin{align*}
\P(n_k^{(S)} \le (n_{\min} - 1) B_{\star}  - \Delta \,|\, R) &\le \exp\left( - \frac{(\E(n_k^{(S)}\,|\, R)- n_{\min} B_\star + \Delta)^2}{2\E(n_k^{(S)}\,|\, S)} \right)\\
& \le \exp\left( - \frac{\Delta^2}{2 n_{\max} B_{\#}} \right).
\end{align*}
Since the right hand side does not depend on $R$, we can take expectations of both sides with respect to $R$ to get
\[
\P(n_k^{(S)} \le (n_{\min} - 1) B_{\star}  - \Delta) \le \exp\left( - \frac{\Delta^2}{2 n_{\max} B_{\#}} \right).
\]
This implies, by an application of the union bound, that
\[
\P(n_{\min}^{(S)} \le (n_{\min} - 1) B_{\star}  - \Delta) \le K\exp\left( - \frac{\Delta^2}{2 n_{\max} B_{\#}} \right).
\] 
Choosing, $\Delta = \sqrt{2r'n_{\max}B_{\#}\log n}$ we conclude, with probability at least $1 - \frac{K}{n^{r'}}$, that
\begin{equation}\label{eq:min_comm_egonet}
n_{\min}^{(S)} \ge (n_{\min} - 1) B_{\star}  - \sqrt{2r'n_{\max}B_{\#}\log n}.
\end{equation}
One can prove similarly that
\[
\P(n_{\max}^{(S)} \ge n_{\max} B_{\#}  + \Delta) \le K\exp\left( - \frac{\Delta^2}{3 n_{\max} B_{\#}} \right),
\] 
and then, taking $\Delta = \sqrt{3r'n_{\max}B_{\#}\log n}$, conclude, with probability at least $1 - \frac{K}{n^{r'}}$, that
\begin{equation}\label{eq:max_comm_egonet}
n_{\max}^{(S)} \le n_{\max} B_{\#}  + \sqrt{3r'n_{\max}B_{\#}\log n}.
\end{equation}

\subsubsection{Analysis of adjacency spectral clustering}
We shall use the community size estimates from the previous subsection along with Lemma~\ref{lem:lei-rinaldo}.

\noindent
\textbf{Random $m$-subgraphs:} 	From (\ref{eq:min_comm_random}) and (\ref{eq:max_comm_random}) we have, with probability at least $1 - \frac{2K}{n^{r'}}$, that
\[
\frac{n^{(S)}_{\max}}{(n^{(S)}_{\min})^2} \le \frac{n_{\max} q + \sqrt{4r' n_{\max} q \log n}}{(n_{\min} q - \sqrt{4r' n_{\max} q \log n})^2} \le \frac{n_{\max}}{n_{\min}^2 q} C_{n,m,r'},
\]
where
\[
C_{n,m,r'} = \left(1 + \sqrt{\frac{4r'\log n}{n_{\max} q}}\right) \left(1 - \sqrt{\frac{4r' n_{\max} \log n}{n_{\min}^2 q}}\right)^{-2}.
\]
Note that $C_{n, m, r'}$ stays bounded, e.g., by $6$, if $\frac{n_{\min}^2 q}{n_{\max} \log n} \ge 16r'$. In fact, it approaches $1$ as $\frac{n_{\min}^2 q}{n_{\max} \log n} \rightarrow \infty$. Thus, from (\ref{eq:lei_rinaldo}), with probability at least $1 - \frac{2K}{n^{r'}}$ over randomness in $S$ and with probability at least $1 - \frac{1}{m^r}$ over randomness in $A_S$, we have
\[
\frac{1}{2}\delta(\hat{Z}_S, Z_S) \le \min \left\{ c^{-1}(2 + \epsilon)\frac{K}{\lambda^2 \alpha_n}\frac{n_{\max}}{n_{\min}^2 q} C_{n,m,r'}, 1 \right\}.
\]
We conclude that
\begin{equation}\label{eq:ASP_random}
\frac{1}{2}\E\delta(\hat{Z}_S,Z_S) \le \min \left\{ c^{-1} (2+\epsilon) \frac{K}{\lambda^2 \alpha_n}\frac{n_{\max}}{n_{\min}^2 q} C_{n,m,r'}, 1 \right\} + \frac{2K}{n^{r'}} + \frac{1}{m^r}.
\end{equation}

\vskip10pt
\noindent
\textbf{$1$-hop ego neighborhoods:} By (\ref{eq:lei_rinaldo}), given $S$ (non-empty), with probability at least $1 - \frac{1}{|S|^r}$ over the randomness in $A_S$, we have
\[
\frac{1}{2}\delta(\hat{Z}_S, Z_S) \le \min \left\{ c^{-1}(2+\epsilon) \frac{K}{\lambda^2 \alpha_n}\frac{n_{\max}^{(S)}}{(n_{\min}^{(S)})^2}, 1 \right\} =: J_S.
\]
So
\[
\frac{1}{2}\E[\delta(\hat{Z}_S,Z_S) \,|\, S] \le J_S + \frac{1}{|S|^r}.
\]
Therefore
\[
\frac{1}{2} \E \delta(\hat{Z}_S,Z_S) \mathbf{1}_{(|S| \ge m_{\star})} \le \E J_S \mathbf{1}_{(|S| \ge m_{\star})} + \E \left(\frac{\mathbf{1}_{(|S| \ge m_{\star})}}{|S|^r}\right).
\]
From (\ref{eq:min_comm_egonet}) and (\ref{eq:max_comm_egonet}) we have, with probability at least $1 - \frac{2K}{n^{r'}}$, that
\[
\frac{n^{(S)}_{\max}}{(n^{(S)}_{\min})^2} \le \frac{n_{\max} B_{\#} + \sqrt{3r' n_{\max} B_{\#} \log n}}{((n_{\min} - 1) B_{\star}  - \sqrt{2r' n_{\max} B_{\#} \log n})^2} \le \frac{n_{\max} B_{\#}}{n_{\min}^2 B_{\star}^2} D_{n,B,r'} = \frac{n_{\max}} {n_{\min}^2 \lambda^2 \alpha_n} D_{n,B,r'},
\]
where
\begin{align*}
D_{n,B,r'} &= \left(1 + \sqrt{\frac{3r'\log n}{n_{\max} B_{\#}}}\right) \left(1 - \frac{1}{n_{\min}} - \sqrt{\frac{2r' n_{\max} B_{\#} \log n}{n_{\min}^2 B_{\star}^2}}\right)^{-2}\\
&= \left(1 + \sqrt{\frac{3r'\log n}{n_{\max}\alpha_n}}\right) \left(1 - \frac{1}{n_{\min}} - \sqrt{\frac{2r' n_{\max}\log n}{n_{\min}^2 \lambda^2 \alpha_n}}\right)^{-2}.
\end{align*}
Note that $D_{n,B,r'}$ stays bounded, e.g., by $16 (1 + \sqrt{\frac{3}{8}})$, if $\frac{n_{\min}^2 \lambda^2 \alpha_n}{n_{\max} \log n} \ge 8r'$. In fact, it approaches $1$ as $\frac{n_{\min}^2 \lambda^2 \alpha_n}{n_{\max} \log n} \rightarrow \infty$.
Thus
\[
\E J_S \mathbf{1}_{(|S| \ge m_{\star})} \le \min\left\{ c^{-1} (2 + \epsilon) \frac{K}{\lambda^4 \alpha_n^2}\frac{n_{\max}}{n_{\min}^2} D_{n,B,r'}, 1 \right\} + \frac{2K}{n^{r'}}.
\]
On the other hand,
\begin{align*}
\E \left(\frac{\mathbf{1}_{(|S| \ge m_{\star})}}{|S|^r}\right) &\le \E \left(\frac{\mathbf{1}_{(|S| \ge 1)}}{|S|^r}\right)\\
&= \E \left(\frac{\mathbf{1}_{(\frac{(n - 1)B_{\star}}{2} > |S| \ge 1)}}{|S|^r}\right) + \E \left(\frac{\mathbf{1}_{(|S| \ge \frac{(n - 1)B_{\star}}{2})}}{|S|^r}\right)\\
&\le \P\left(|S| < \frac{(n-1)B_{\star}}{2}\right) + \left(\frac{2}{(n - 1)B_{\star}}\right)^r\\
&\le \exp\left(-\frac{nB_{\star}^2}{16 B_{\#}}\right) + \left(\frac{2}{(n - 1)B_{\star}}\right)^r \hskip20pt \text{(by (\ref{eq:lower_tail_bound_on_ego_nbd_size}))}\\
&\le \frac{16 B_{\#}}{n B_{\star}^2} + \left(\frac{4}{n B_{\star}}\right)^r = \frac{16}{n \lambda^2 \alpha_n} + \left(\frac{4}{n \lambda \alpha_n}\right)^r.
\end{align*}
So, finally, we have
\begin{equation}\label{eq:ASP_ego}
\frac{1}{2} \E \delta(\hat{Z}_S,Z_S) \mathbf{1}_{(|S| \ge m_{\star})} \le \min\left\{ c^{-1} \frac{K}{\lambda^4 \alpha_n^2}\frac{n_{\max}}{n_{\min}^2} D_{n,B,r'}, 1 \right\} + \frac{2K}{n^{r'}} + \frac{16}{n \lambda^2 \alpha_n} + \left(\frac{4}{n \lambda \alpha_n}\right)^r.
\end{equation}

\subsubsection{Analysis of SDP}
In the setting of Corollary~\ref{cor:sdp}, for a random $m$-subgraph $S$ we have $\min_{k} B_{kk} \ge \frac{a}{n} = \frac{\tilde{a}}{m}, \max_{k\ne k'} B_{kk'} \le \frac{b}{n} = \frac{\tilde{b}}{m}$ where $\tilde{a} = \frac{am}{n}$, $\tilde{b} = \frac{bm}{n}$. Let $n_{kk'}^{(S)}$ denote the number of pairs of vertices in the subgraph $S$ such that one of them is from community $k$ and the other is from community $k'$. Now by Lemma~\ref{lem:hyper_lower}
\[
\P(n_k^{(S)} \ge c n_k m/n) \le e^{-(c - 1)^2m\pi_{\min}/4}.
\]
So with probability at least $1 - Ke^{-(c - 1)^2m\pi_{\min}/4}$ we have, for all $1 \le k \le k' \le K$, that
\[
	n_{kk'}^{(S)} = \begin{cases}
	n_{k}^{(S)} n_{k'}^{(S)} \le c^2 n_k n_k' m^2/n^2 \le c_1 n_{kk'} m^2/n^2, & \text{ for } k < k',\\
	\frac{1}{2}n_{k}^{(S)} (n_{k}^{(S)} - 1) \le \frac{1}{2}c n_k m/n(c n_k m / n -1) \le c_1 n_{kk} m^2/n^2, & \text{ for } k = k',
	\end{cases}
\]
where $c_1 = c^2 \left(1 + \frac{\frac{cm}{n}-1}{\frac{cm}{n}(n\pi_{\min} - 1)}\right)$. Similarly, we can show that, with probability at least $1 - K e^{-(1-c')^2m\pi_{\min}/4}$
\[
n_{kk'}^{(S)} \ge \begin{cases}
	c_2 n_{kk'} m^2/n^2, & \text{ for } k < k',\\
	c_2 n_{kk} m^2/n^2, & \text{ for } k = k',
	\end{cases}
\]
where $c_2 = (c')^2 \left(1 + \frac{\frac{c'm}{n}-1}{\frac{c'm}{n}(n\pi_{\max} - 1)}\right)$. Let $\mathcal{G} = \{S \mid c_2 n_{kk'} m^2/n^2 \le n_{kk'}^{(S)} \le c_1 n_{kk'} m^2/n^2, \, \text{ for all } 1\le k \le k' \le K\}$. Given $S$, the expected variance of edges in $S$ is $\frac{2}{m(m-1)}\sum_{1 \le k \le k' \le K}B_{kk'}(1 - B_{kk'})n_{kk'}^{(S)} =: \frac{\tilde{g}}{m}$, say. Then, if $S \in \mathcal{G}$, we have
\[
\frac{c_2g}{n} \le \frac{c_2 m(n-1)}{(m - 1)n} \times \frac{g}{n} \le \frac{\tilde{g}}{m} \le \frac{c_1 m(n-1)}{(m - 1)n} \times \frac{g}{n} = \frac{m}{n^2} \times \bar{g}.
\]
and so, using our assumptions on $a$, $b$, $g$, we get that $\tilde{g} \ge 9$ and $\tilde{a}, \tilde{b}$ satisfy $(\tilde{a} - \tilde{b})^2 \ge 484 \epsilon^{-2} \frac{m^2}{n^2}\bar{g} \ge 484 \epsilon^{-2} \tilde{g}$. Therefore, using Lemma~\ref{lem:sdp-gv} on the subgraph $S$, we conclude that, given $S \in \mathcal{G}$, we have
\[
\frac{1}{m^2}\|\hat{C}^{(S)} - \hat{C}^{(S)}\|_F^2 \le \epsilon,
\]
with probability at least $1 - e^3 5^{-m}$, where $\hat{C}^{(S)}$ is a solution of SDP-GV on the subgraph $S$. As $\P(S\notin \mathcal{G}) \le Ke^{-(c - 1)^2 m\pi_{\min}/4} + Ke^{-(1 - c')^2 m\pi_{\min}/4}$, we conclude that
\[
\E \tilde{\delta}(\hat{C}^{(S)},C^{(S)}) \le \epsilon + e^3 5^{-m} + Ke^{-(c - 1)^2 m\pi_{\min}/4} + Ke^{-(1-c')^2 m\pi_{\min}/4}.
\] 

\subsection{Proofs of some results for \gale{}}
In this section we collect the proofs of the auxiliary results we presented in Section~\ref{sec:gale-analysis}. 

\begin{proof}[Proof of Lemma~\ref{lem:nicount}]
	Note that the $N_i$'s are i.i.d. {\sf Binomial}$(T; m/n)$ random variables and hence by Chernoff's inequality we have,
	\begin{equation}\label{eq:ni-chernoff1}
		\P(N_i < \tau) \le e^{-(1 - \theta)^2 Tm/2n}.
	\end{equation}
	Another use of Chernoff gives, for $\delta > 0$, 
	\[
	\P\left(\sum_{i}\mathbf{1}_{\{ N_i < \tau\}} > n \P(N_1 < \tau)(1 + \delta) \right)\le \exp(-\delta^2 n \P(N_1< \tau)/3).
	\]
	Setting $\delta = \sqrt{\frac{3r\log n}{n \P(N_1 < \tau)}}$, we conclude that with probability at least $1 - \frac{1}{n^r}$,
	\begin{align*}
	\sum_{i}\mathbf{1}_{\{ N_i < \tau\}} &\le n \P(N_1 < \tau) + \sqrt{3r n \P(N_1 < \tau) \log n}\\
	&\le ne^{-(1 - \theta)^2 Tm/2n} \left(1 + \sqrt{\frac{3r \log n}{ne^{-(1 - \theta)^2 Tm/2n}}}\right).
	\end{align*}
\end{proof}

\begin{proof}[Proof of Lemma~\ref{lem:erygraph}]
	We will first show that for three different nodes $a, b, c$ in $\mathcal{S}_{m, T}$ the overlap variables $Y_{ab}$ and $Y_{ac}$ are independent. This is an immediate consequence of the following two facts:
	\begin{enumerate}
		\item[(i)] $Y_{ab}$ and $y^{(a)}$ are independent which follows from the observation that for any $m$-subset $S$ of $[n]$, $Y_{ab} \mid y^{(a)} = \mathbf{1}_{S} \sim$ {\sf Hypergeometric}$( m, m,n)$.
		\item[(ii)] Given $y^{(a)}$, the overlaps $Y_{ab}$ and $Y_{ac}$ are independent, which follows from the fact that $y^{(b)}$ and $y^{(c)}$ are independent.
	\end{enumerate}
		
	From the discussion above, $Y_{ab} \sim$ {\sf Hypergeometric}$(m, m, n)$. Using Lemma~\ref{lem:hyper_lower}  we get, with $\epsilon=1/2$, we get that
	\[
		\P(a \sim b) = \P\left(Y_{ab} \geq \thresh \right) \ge 1 - \exp \left(-\frac{m^2}{16n}\right).
	\]
\end{proof}

\begin{proof}[Proof of Lemma~\ref{lem:cluster_presence}]
Notice that $n_{k}^{(S)} \mid Y_{ab} \sim {\sf Hypergeometric}(Y_{ab}, n\pi_k, n)$, which follows from the fact that given $Y_{ab}$, $S_a \cap S_b$ is distributed as a uniform $Y_{ab}$-subset of $[n]$\footnote{Thanks to Satyaki Mukherjee for making this observation, which makes the proof considerably shorter than our original approach based on concentration inequalities for sums of dependent Bernoulli random variables.}. So, Lemma~\ref{lem:hyper_lower} gives us
\[
	\P\left(n_k^{(S)} < Y_{ab}\pi_k \left(1 - \sqrt{\frac{4 r \log n}{Y_{ab}\pi_k}} \right) \big| \, Y_{ab}\right) \le \frac{1}{n^r}.
\]
Therefore the same bound holds for the unconditional probability. Another application of Lemma~\ref{lem:hyper_lower} for $Y_{ab} \sim {\sf Hypergeometric}(m, m, n)$ gives us
\[
	\P\left(Y_{ab} < \frac{m^2}{n} \left(1 - \sqrt{\frac{4 r n\log n}{m^2}} \right) \right) \le \frac{1}{n^r}.
\]
Using these two bounds, we conclude that with probability at least $1 - \frac{2}{n^r}$, we have that
\[
n_k^{(S)} \ge \frac{m^2\pi_k}{n} \left(1 - O\left(\sqrt{\frac{n\log n}{m^2 \pi_k}}\right) \right),
\]
as long as $\frac{m^2 \pi_k}{n\log n}$ is large enough ($\frac{m^2 \pi_k}{n\log n} \ge 20r$ suffices to make the RHS above $\ge 0$).
\end{proof}

%\rd
%\section{Extra simulations}
%We did not report Monte Carlo errors for \pace{} and \gale{} in our simulations, we should do that.
%\bk
\end{document}